%% file: neurips_2026.tex
\documentclass{article}

\usepackage[nonatbib, final, main]{neurips_2026}

\usepackage[utf8]{inputenc} 
\usepackage[T1]{fontenc}    
\usepackage{hyperref}       
\hypersetup{
    colorlinks=false,
    pdfborder={0 0 1},
    pdfborderstyle={/S/U/W 1},
    linkbordercolor={0 0 1},
    citebordercolor={0 0 1},
    urlbordercolor={0 0 1}
}

\usepackage{url}            
\usepackage{booktabs}       
\usepackage{amsfonts}       
\usepackage{nicefrac}       
\usepackage{microtype}      
\usepackage{xcolor}         
\usepackage[ruled,linesnumbered]{algorithm2e} 
\usepackage{placeins}
\usepackage[font=small]{caption}
\usepackage{subcaption}  
\usepackage{graphicx}
\usepackage{multirow}
\usepackage{float}
\usepackage{amsmath}
\usepackage{amssymb}
\usepackage{mathtools}
\usepackage{amsthm}
\usepackage{enumitem}
\usepackage{multicol}
\usepackage{multirow}
\usepackage[table]{xcolor}
\usepackage{tabularx}
\usepackage{booktabs}

\usepackage[capitalize,noabbrev]{cleveref}
\usepackage{tikz}
\usetikzlibrary{arrows.meta,positioning,fit,calc,matrix}

\usepackage{algorithmic}
\usepackage{varwidth}

\usepackage{amsfonts}       
\usepackage{bm}
\usepackage{bbm}

\DeclareMathOperator*{\argmin}{arg\,min}
\newtheorem{theorem}{Theorem}[section]
\newtheorem{proposition}[theorem]{Proposition}
\newtheorem{lemma}[theorem]{Lemma}

\newtheorem{assumption}[theorem]{Assumption}

\usepackage[most]{tcolorbox}
\usepackage{tikz}
\usepackage{fontawesome5}
\usetikzlibrary{arrows.meta, calc, decorations.markings, angles, quotes}

\definecolor{cReach}{HTML}{1F7A6D}
\definecolor{cPerp}{HTML}{C75D3C}
\definecolor{cTotal}{HTML}{2B5A9E}
\definecolor{cTrade}{HTML}{8B5BAA}
\definecolor{cTarget}{HTML}{1A4F6E}
\definecolor{cMuted}{HTML}{6B7280}
\definecolor{cFaint}{HTML}{9CA3AF}

\usepackage[textsize=tiny]{todonotes}

\newcommand{\R}{\mathbb{R}}

\newcommand{\calL}{\mathcal{L}}
\newcommand{\calQ}{\mathcal{Q}}

\title{CoreQ: Mismatch-Corrected Successive\\ Rounding for LLM Quantization
}

\author{%
Seohyeon Cha$^{1}$ \quad Huancheng Chen$^{1}$ \quad Dongjun Kim$^{1}$ \quad Haoran Zhang$^{1}$ \\
\textbf{Kevin Chan}$^2$ \quad \textbf{Gustavo de Veciana}$^1$ \quad \textbf{Haris Vikalo}$^1$ \\
$^1$The University of Texas at Austin \quad $^2$DEVCOM Army Research Laboratory\\
}

\begin{document}

\maketitle
\vspace{-0.1 in}

\begin{abstract}

Post-training quantization (PTQ) enables efficient deployment of large language models by mapping pretrained weights to low-bit formats without retraining, typically using a small calibration set to minimize a layer-wise calibration objective. However, this sequential procedure induces a mismatch: errors from earlier quantized layers alter the inputs received by later layers, causing the activations to deviate from those of the full-precision model. Recent approaches introduce mismatch-aware calibration objectives to compensate for this effect, but leave open how much of the observed mismatch should shift each layer's calibration target.  We show that fully applying the empirical mismatch correction can
overfit finite calibration data, while using a fixed correction
coefficient ignores variation in mismatch estimates across layers. 
To address these limitations, we propose CoreQ, a learning-free PTQ framework that uses a geometric decomposition of the mismatch to compute a closed-form, layer-adaptive correction coefficient. This coefficient controls how much of the observed mismatch enters each layer's calibration target, reducing overfitting to limited calibration data without coefficient tuning. CoreQ then applies efficient greedy successive rounding to the resulting triangular least-squares objective. Its bounded beam-search extension, K-CoreQ, trades modest additional compute for improved performance.
Across multiple LLM families, scales, bit-widths, and quantization settings, CoreQ and K-CoreQ achieve lower perplexity and higher downstream accuracy than strong PTQ baselines. Code is available at
\href{https://github.com/seohyeon-cha/CoreQ}{\faGithub\ CoreQ}.

\end{abstract}
\vspace{-0.1 in}
  
\input{main/intro}

\input{main/background}
\input{main/obj_ver2}
\input{main/result}

\input{main/related_work}
\input{main/conclusion}

\section*{Acknowledgment}
This paper was supported in part by a DOD Award  W911NF-24-2-0185 and NSF Award 2148224 and is supported in part by funds from OUSD R\&E, NIST, and industry partners as specified in the Resilient \& Intelligent NextG Systems (RINGS) program. We thank Kyoungpyo Lee for many insightful discussions that helped shape this work throughout its development.

\newpage
\bibliography{reference}
\bibliographystyle{reference}
\newpage

\newpage
\appendix
\input{appendix/appendix}


\end{document}

%% file: main/intro.tex
\section{Introduction}

Large language models (LLMs) achieve strong performance across a wide range of tasks, but their size presents a major challenge to practical deployment \cite{ge2023openagi,yang2024harnessing,nam2024using}. Modern transformer-based LLMs contain billions of parameters, requiring tens to hundreds of gigabytes of weight storage when stored in high precision and incurring substantial inference latency due to memory traffic \cite{gholami2024ai,kim2023squeezellm}. These costs limit deployment on commodity hardware and make on-device or edge inference difficult \cite{xu2024device}. Consequently, reducing memory footprint and runtime while preserving accuracy has become essential for practical LLM deployment.

Post-training quantization (PTQ) enables compression of a pretrained model without retraining by mapping its weights to low-bit formats, providing a practical path to efficient inference. To remain scalable, PTQ is typically performed layer-by-layer. Given a small calibration set, sequential PTQ quantizes each layer by solving a discrete rounding problem that minimizes a layer-wise calibration objective using activations from the already-quantized prefix \cite{nagel2020up,hubara2021accurate,li2021brecq,frantar2022obq,frantar2022gptq}. While efficient, this layer-wise procedure ignores the compounding effect of earlier quantization errors across layers. As layers are quantized sequentially, the activations entering deeper layers deviate from those in the full-precision model, creating a mismatch between the full-precision and partially quantized computation paths \cite{ligptaq,arai2026qep,liao2024apiq}.

Recent methods, such as GPTAQ \cite{ligptaq} and QEP \cite{arai2026qep}, address this problem by redefining the calibration objective: they use the full-precision layer output on full-precision activations as the calibration target, while evaluating the quantized layer on activations produced by the quantized prefix. This incorporates the upstream-induced mismatch into layer-wise PTQ, but leaves open how strongly the observed mismatch should shift the calibration target under finite calibration data. Applying the full correction can overfit calibration data, while using a fixed correction coefficient cannot adapt to variation across layers and model families. Other approaches improve the layer-wise proxy through richer curvature estimates \cite{kimguidedquant,tseng2025model,kim2024boa} or by quantizing multiple layers jointly \cite{ding2023cbq}, but these strategies either increase the cost of the layer-wise objective or change the quantization unit.

Even with a well-chosen calibration objective, PTQ still requires solving a \emph{discrete} optimization problem over low-bit weight grids. Most methods obtain quantized weights by sequentially rounding the weight matrix, either through greedy column-wise rounding \cite{frantar2022gptq,chee2023quip,chen2025geometry,zheng2025foem} or via iterative refinement schemes such as coordinate descent \cite{behdin2023quantease,nair2024cdquant}. The rounding strategy directly affects quantization performance: poor early decisions can propagate across weight columns and substantially degrade the quality of the solution. Thus, improving PTQ requires not only a well-defined calibration objective, but also a more reliable way to search over the low-bit weight space. 


To address these limitations, we propose \textbf{CoreQ}, a learning-free PTQ framework that adaptively controls how strongly mismatch correction shifts each layer's calibration target. CoreQ computes a closed-form correction coefficient, \(\alpha_{\mathrm{corr}}\), by decomposing the mismatch into a  component reachable by changing the current layer's weights and an orthogonal residual. This provides a layer-adaptive alternative to full mismatch correction and fixed scaling, without model-specific tuning, gradient updates, or validation search. We then solve the resulting discrete rounding problem using successive-rounding algorithms. While CoreQ provides a greedy solution, \textbf{K-CoreQ} extends it with bounded beam search by maintaining multiple partial rounding paths. This expanded search reduces error accumulation from early weight choices and provides a practical trade-off between computational efficiency and quantization quality.

Our contributions are summarized as follows:
\vspace{-0.03 in}

\begin{itemize}[leftmargin=*,noitemsep]
\item \textbf{Finite-calibration analysis.} We show theoretically and empirically that fully compensating for mismatch on finite calibration data leads to overfitting (Section~\ref{sec:calibration-target} and Appendix~\ref{app:alpha-analysis}). The analysis motivates a flexible, data-dependent mismatch correction strategy. 

\item \textbf{Layer-adaptive mismatch correction.}
We introduce CoreQ, a learning-free PTQ framework that computes a layer-wise correction coefficient with no need for gradient updates or validation search.

\item \textbf{Successive rounding.}
Given this correction coefficient, CoreQ solves the rounding problem induced by the mismatch-corrected calibration target by exploiting the triangular structure of a least-squares proxy. This yields an efficient greedy successive-rounding solver and a bounded beam-search extension, $K$-CoreQ, which provides a controllable performance-cost trade-off (Section~\ref{subsec:autoregressive-decoding}).

\item \textbf{Extensive empirical validation.}
Across multiple LLM families, model scales, bit-widths, and quantization settings, CoreQ and $K$-CoreQ improve perplexity and downstream task accuracy over strong PTQ baselines. For example, on 2-bit quantized LLaMA-2-70B, we achieve a \textbf{48--56\%} improvement in WikiText-2 perplexity relative to the best baseline (Section~\ref{sec:results}).

\end{itemize}

\vspace{-0.05 in}

%% file: main/background.tex
\section{Background: Calibration Objectives for Layer-Wise PTQ}
\label{sec:background}

We consider sequential, layer-wise post-training quantization (PTQ) of a transformer model. All quantities below refer to a single layer; we suppress the layer index \(l\) for clarity. Let \(\mathbf W\in\mathbb R^{m\times n}\) denote the full-precision weight matrix of the current layer, and let \(\mathbf Q\in\mathcal Q\) denote its quantized counterpart on a discrete low-bit grid. 

Let
\(\mathbf X_{\mathrm q}
= [x_{\mathrm q}^{(1)},\ldots,x_{\mathrm q}^{(N)}] \in\mathbb R^{n\times N}\)
denote the matrix of \(N\) calibration activation vectors entering the current layer after all preceding layers have been quantized. The layer-wise PTQ problem can be written in the common form
\(
    \min_{\mathbf Q\in\mathcal Q}
    \|\mathbf Y-\mathbf Q\mathbf X_{\mathrm q}\|_F^2,
\)
where \(\mathbf Y\) is the \emph{calibration target}, i.e., the desired output that the quantized layer is asked to match on the quantized-prefix activation \(\mathbf X_{\mathrm q}\).

\paragraph{Standard Calibration Objective.}
Standard PTQ methods use the full-precision output on the same quantized-prefix activations as the calibration target, i.e., \(\mathbf Y=\mathbf W\mathbf X_{\mathrm q}\). This gives the standard calibration objective \cite{nagel2020up}
\begin{equation}
  \mathcal L_{\mathrm{std}}(\mathbf Q;\mathbf X_{\mathrm q})
  :=
  \|(\mathbf W-\mathbf Q)\mathbf X_{\mathrm q}\|_F^2
  =
  \mathrm{tr}\!\big(
  (\mathbf W-\mathbf Q)\mathbf H
  (\mathbf W-\mathbf Q)^\top
  \big),
  \label{eq:std-cal-loss}
\end{equation}
where \(\mathbf H:=\mathbf X_{\mathrm q}\mathbf X_{\mathrm q}^{\top}\) serves as the Hessian proxy. Solvers such as GPTQ \cite{frantar2022gptq} and LDLQ \cite{chee2023quip} approximately minimize \eqref{eq:std-cal-loss} by sequential rounding guided by this proxy \cite{kimguidedquant,tseng2025model}.

\paragraph{Mismatch-Aware Calibration Objective.}
Let \(\mathbf X_{\mathrm f}\in\mathbb R^{n\times N}\) denote the full-precision activations corresponding to the same calibration tokens as \(\mathbf X_{\mathrm q}\). Since the standard calibration target is built from \(\mathbf X_{\mathrm q}\), it does not include the upstream-induced \emph{mismatch}
\[
    \mathbf D
    :=
    \mathbf W(\mathbf X_{\mathrm f}-\mathbf X_{\mathrm q}),
\]
which captures the effect of preceding quantized layers at the current layer \cite{ding2023cbq,liao2024apiq}. Mismatch-aware calibration incorporates this mismatch by using the target \(\mathbf Y=\mathbf W\mathbf X_{\mathrm f}\). The resulting mismatch-aware calibration objective is \cite{ligptaq,arai2026qep}
\begin{equation}
  \mathcal L_{\mathrm{mis}}(\mathbf Q;\mathbf X_{\mathrm f},\mathbf X_{\mathrm q})
  :=
  \|\mathbf W\mathbf X_{\mathrm f}-\mathbf Q\mathbf X_{\mathrm q}\|_F^2
  =
  \|(\mathbf W-\mathbf Q)\mathbf X_{\mathrm q}+\mathbf D\|_F^2.
  \label{eq:mismatch-rec-loss}
\end{equation}
Thus, both objectives optimize the quantized layer on \(\mathbf X_{\mathrm q}\), but differ in the calibration target: standard calibration uses \(\mathbf Y=\mathbf W\mathbf X_{\mathrm q}\), while mismatch-aware calibration uses \(\mathbf Y=\mathbf W\mathbf X_{\mathrm f}\).

%% file: main/obj_ver2.tex
\section{CoreQ: Mismatch-Corrected Calibration and Successive Rounding}
\label{sec:methodology}

Section~\ref{sec:background} introduced two endpoint calibration objectives: standard calibration, which uses the target $\mathbf W\mathbf X_{\mathrm q}$, and mismatch-aware calibration, which uses the target $\mathbf W\mathbf X_{\mathrm f}$. We start from the observation that neither endpoint is generally reliable: the former ignores the mismatch, while the latter can overfit finite calibration data. CoreQ addresses this by constructing an interpolated calibration target with a closed-form correction coefficient, avoiding manual tuning or validation search.

\subsection{Calibration Target Selection}
\label{sec:calibration-target}

\input{files/framework}

\paragraph{Orthogonal decomposition of the mismatch.} 
From \eqref{eq:mismatch-rec-loss}, the current layer can offset the mismatch $\mathbf{D}$ only through the term $(\mathbf{W}-\mathbf{Q})\mathbf{X}_\mathrm{q}$. Since $\mathbf{X}_\mathrm{q}$ is fixed during layer-wise calibration, all output changes realizable by changing the current layer's weights lie in the \emph{reachable subspace}
$\mathcal{R} := \{\mathbf{A}\mathbf{X}_\mathrm{q} : \mathbf{A}\in\R^{m\times n}\}$.
The mismatch $\mathbf{D}$ need not lie entirely in this subspace. Projecting it onto $\mathcal{R}$ yields the orthogonal decomposition
\begin{equation}
\mathbf{D}
=
\widetilde{\mathbf{S}}\mathbf{X}_\mathrm{q}
+
\boldsymbol{\eta},
\quad
\widetilde{\mathbf{S}}
:=
\mathbf{D}\mathbf{X}_\mathrm{q}^\top
(\mathbf{X}_\mathrm{q}\mathbf{X}_\mathrm{q}^\top)^{-1},
\quad
\boldsymbol{\eta}\perp\mathcal{R},
\label{eq:U-decomp}
\end{equation}
where $\widetilde{\mathbf{S}}\mathbf{X}_\mathrm{q}\in\mathcal{R}$ is the part of the mismatch absorbable by changing the current layer's weights, while $\boldsymbol{\eta}$ is the residual component.

Substituting this decomposition into \eqref{eq:mismatch-rec-loss} separates the reachable component from the residual:
\begin{equation}
  \calL_{\mathrm{mis}}(\mathbf{Q})
  =
  \|(\mathbf{W}+\widetilde{\mathbf{S}})\mathbf{X}_\mathrm{q}
  -\mathbf{Q}\mathbf{X}_\mathrm{q}\|_F^2
  +
  \|\boldsymbol{\eta}\|_F^2 .
  \label{eq:asym-split}
\end{equation}
Since the residual term is independent of $\mathbf{Q}$, minimizing $\calL_{\mathrm{mis}}$ is equivalent to using the \emph{reachable calibration target}
$\mathbf{Y}_{\mathcal R}:=(\mathbf{W}+\widetilde{\mathbf{S}})\mathbf{X}_\mathrm{q}$.
However, $\widetilde{\mathbf{S}}$ is estimated from finite calibration data, so using $\mathbf{Y}_{\mathcal R}$ without shrinkage can overfit finite-calibration error.

\input{files/alpha_sweep}

\paragraph{Choosing the correction coefficient.}
To control reliance on this finite-calibration estimate, we consider calibration targets along the path
$\mathbf{Y}^{(\alpha)} := \mathbf{W}\mathbf{X}^{(\alpha)}$, where
$\mathbf{X}^{(\alpha)} := (1-\alpha)\mathbf{X}_{\mathrm q}+\alpha \mathbf{X}_{\mathrm f}$
for $\alpha\in[0,1]$. Using
$\mathbf{D}=\widetilde{\mathbf{S}}\mathbf{X}_{\mathrm q}+\boldsymbol{\eta}$,
this target can be written as
\begin{equation}
    \mathbf{Y}^{(\alpha)}
    =
    \mathbf{W}\mathbf{X}_{\mathrm q}+\alpha \mathbf{D}
    =
    (\mathbf{W}+\alpha\widetilde{\mathbf{S}})\mathbf{X}_{\mathrm q}+\alpha\boldsymbol{\eta}.
    \label{eq:WX-alpha}
\end{equation}
Here, the \emph{correction coefficient} $\alpha$ controls how far the calibration target moves from standard calibration toward mismatch-aware calibration: $\alpha=0$ gives \eqref{eq:std-cal-loss}, while $\alpha=1$ gives \eqref{eq:mismatch-rec-loss}. The next theorem formalizes the trade-off between ignoring the mismatch and fully incorporating its finite-calibration estimate, showing that the population-optimal coefficient is generally smaller than $1$ under finite-calibration error.

\vspace{0.12in}
\begin{theorem}[Intermediate correction under finite-calibration error]
\label{thm:population-alpha-main}
Let $\bar{\mathbf H}:=\mathbb E[X_{\mathrm q}X_{\mathrm q}^\top]\succ0$ and
$\bar{\mathbf S}:=\mathbf{W}\mathbb E[(X_{\mathrm f}-X_{\mathrm q})X_{\mathrm q}^\top]\bar{\mathbf H}^{-1}$. The population
layer-wise error
$\mathcal L_{\mathrm{pop}}(\widehat{\mathbf{W}}):=
\mathbb E\|\mathbf{W}X_{\mathrm f}-\widehat{\mathbf{W}} X_{\mathrm q}\|_2^2$
is minimized over continuous weights by \(\mathbf{W}^\star=\mathbf{W}+\bar{\mathbf S}\).
Assume \(\mathbb E[\widetilde{\mathbf S}]=\bar{\mathbf S}\). 
For continuous weights of the form $\mathbf{W}+\alpha\widetilde{\mathbf S}$, the expected excess population error is
\[
    R(\alpha)
    :=
    \mathbb E_{\widetilde{\mathbf S}}
    \!\left[
    \mathcal L_{\mathrm{pop}}(\mathbf{W}+\alpha\widetilde{\mathbf S})
    -
    \mathcal L_{\mathrm{pop}}(\mathbf{W}^\star)
    \right]=(1-\alpha)^2 \mu^2+\alpha^2 \sigma^2.
\]
Therefore, the optimal correction coefficient is
\(
    \alpha^\star
    :=
    \arg\min_{\alpha\in[0,1]}R(\alpha)
    =
    \frac{\mu^2}{\mu^2+\sigma^2},
\)
where
\(\mu^2:=\|\bar{\mathbf S}\bar{\mathbf H}^{1/2}\|_F^2\) is the ideal correction
energy and
\(\sigma^2:=\mathbb E_{\widetilde{\mathbf S}}
\|(\widetilde{\mathbf S}-\bar{\mathbf S})\bar{\mathbf H}^{1/2}\|_F^2\) is the finite estimation error energy.
Therefore, \(0<\alpha^\star<1\) whenever \(\mu^2>0\) and \(\sigma^2>0\).
\end{theorem}


The proof is provided in Appendix~\ref{app:theory}.
\cref{thm:population-alpha-main} is an idealized result:
under unbiased correction estimation, it characterizes the trade-off between applying too little correction and amplifying estimation error from finite calibration data. The optimal coefficient depends on the correction and estimation error energies of each layer, motivating
a correction strength that adapts across layers. The theorem assumes an unbiased estimator $\mathbb E[\widetilde{\mathbf S}]=\bar{\mathbf S}$, which does not hold exactly after matrix inversion. Appendix~\ref{app:biased-correction} extends the
analysis to biased estimators and shows that the resulting deviation vanishes at rate $O(N^{-1})$ in the calibration size. In either case, however, $\mu^2$ and $\sigma^2$ are population quantities that cannot be evaluated during quantization.

CoreQ therefore uses a criterion motivated by the mismatch decomposition in \eqref{eq:U-decomp}. For each layer, CoreQ chooses $\alpha$ to minimize the squared distance between $\mathbf{Y}^{(\alpha)}$ and the reachable calibration target $\mathbf{Y}_{\mathcal R}$, as illustrated in \cref{fig:framework1}(b). Writing $\mathbf{s}:=\widetilde{\mathbf{S}}\mathbf{X}_{\mathrm q}$, the squared distance is 
\begin{equation}
  \|\mathbf{d}(\alpha)\|_F^2
  =
  (1-\alpha)^2\|\mathbf{s}\|_F^2
  +
  \alpha^2\|\boldsymbol{\eta}\|_F^2 .
  \label{eq:distance}
\end{equation}
This criterion selects the correction coefficient by balancing two effects: the first term penalizes leaving the reachable component of the mismatch uncorrected, while the second penalizes the residual component, which cannot be corrected by changing the current layer's weights. Minimizing this distance gives the closed-form correction coefficient
\begin{equation}
  \alpha_{\mathrm{corr}}
  =
  \frac{\|\mathbf{s}\|_F^2}
       {\|\mathbf{s}\|_F^2+\|\boldsymbol{\eta}\|_F^2}
  =
  \frac{\|\widetilde{\mathbf{S}}\mathbf{X}_{\mathrm q}\|_F^2}
       {\|\widetilde{\mathbf{S}}\mathbf{X}_{\mathrm q}\|_F^2+\|\boldsymbol{\eta}\|_F^2}.
  \label{eq:alpha-corr}
\end{equation}
In this form, $\|\mathbf{s}\|_F^2$ plays the role of the correction-benefit term in Theorem~\ref{thm:population-alpha-main}, while $\|\boldsymbol{\eta}\|_F^2$ provides a geometric reliability penalty in place of the unobservable estimation error. Thus, $\alpha_{\mathrm{corr}}$ is a computable analogue of the correction--reliability trade-off that can be evaluated from calibration statistics alone (Lemma~\ref{lem:d-proxy}). \cref{fig:alpha_analysis} and Appendix~\ref{app:alpha-analysis} show that $\alpha_{\mathrm{corr}}$ varies systematically across layers, reducing the correction when the residual component is large.

\subsection{Successive Rounding with the Corrected Target}
\label{subsec:autoregressive-decoding}

With $\alpha_{\mathrm{corr}}$ selected, CoreQ uses the mismatch-corrected calibration target $\mathbf{Y}^{(\alpha_{\mathrm{corr}})}$ to choose the quantized weights $\mathbf{Q}$ over the low-bit grid. With $\mathbf{W}_{\mathrm{corr}}:=\mathbf{W}+\alpha_{\mathrm{corr}}\widetilde{\mathbf{S}}$, 
optimizing $\mathbf Q$ with the corrected target is equivalent to the Hessian-weighted rounding problem
\begin{equation}
  \argmin_{\mathbf{Q}\in\mathcal Q}
  \|\mathbf{Y}^{(\alpha_{\mathrm{corr}})}-\mathbf{Q}\mathbf{X}_{\mathrm q}\|_F^2
  =
  \argmin_{\mathbf{Q}\in\mathcal Q}
  \bigl\|
  (\mathbf{W}_{\mathrm{corr}}-\mathbf{Q})\mathbf{L}
  \bigr\|_F^2,
  \quad
  \mathbf{H}:=\mathbf{X}_{\mathrm q}\mathbf{X}_{\mathrm q}^{\top}
  =
  \mathbf{L}\mathbf{L}^{\top}.
  \label{eq:shifted-rounding}
\end{equation}
where $\mathbf{L}$ is the lower-triangular matrix obtained from Cholesky decomposition of the Hessian. 

\paragraph{Triangular discrete least squares.}
To solve the problem in \eqref{eq:shifted-rounding}, we exploit its triangular least-squares form. Let $\mathbf{R} := \mathbf{L}^\top$. For row $i$, define the transformed target $\mathbf{y}_i := \mathbf{R}\,\mathbf{W}_{\mathrm{corr}, i, :}^\top$ and the quantized row vector $\mathbf{q}_i := \mathbf{Q}_{i, :}^\top$. The Frobenius objective decomposes across rows as $\sum_{i=1}^m \|\mathbf{R}\mathbf{q}_i - \mathbf{y}_i\|_2^2$.
Since $\mathbf{R}$ is upper triangular, each row can be rounded sequentially in reverse order. Suppressing the row index $i$, given a fixed suffix ${q}_{j+1:n}$, define the suffix-conditioned center $c_j$ and incremental cost $\Delta_j$ as
\begin{equation}
  c_j({q}_{j+1:n})
  := \frac{y_j - \sum_{k=j+1}^n \mathbf{R}_{j, k}\, q_k}{\mathbf{R}_{j, j}},
  \qquad
  \Delta_j(q_j \mid {q}_{j+1:n})
  := \mathbf{R}_{j, j}^2\, |q_j - c_j({q}_{j+1:n})|^2.
  \label{eq:triangular-defs}
\end{equation}
The row-wise objective then decomposes as $\|\mathbf{R}\mathbf{q} - \mathbf{y}\|_2^2 = \sum_{j=1}^n \Delta_j(q_j \mid {q}_{j+1:n})$.

\paragraph{Greedy decoding.}
We approximately solve this discrete optimization problem using greedy successive rounding. At step $j$, given the already rounded suffix ${q}_{j+1:n}$, \textbf{CoreQ} rounds the conditioned center $c_j$ to the nearest available quantization level in the grid $\calQ_j$, i.e., it finds
\begin{equation}
  q_j^\star({q}_{j+1:n}) \in \arg\min_{q \in \calQ_j} |q - c_j({q}_{j+1:n})|^2.
  \label{eq:greedy-step}
\end{equation}
As each weight is quantized, its quantization error shifts the centers of the remaining coordinates yet to be rounded. CoreQ recovers LDLQ when $\alpha_{\mathrm{corr}}=0$, while $K$-CoreQ further exploits the same triangular structure via bounded beam search in the following.

\paragraph{Bounded tree search: $K$-CoreQ.}
While efficient, greedy successive rounding may suffer from propagation of rounding errors because early errors affect later conditioned centers through the off-diagonal entries of $\mathbf{R}$ \cite{guo2006sd-beam}. To reduce this effect, we propose \textbf{$K$-CoreQ}, which replaces the single greedy path with a bounded beam search that keeps multiple rounding paths.


We view the rounding of each row as a tree search of depth $n$. At coordinate $j$, each beam state stores the already chosen suffix ${q}_{j+1:n}$ and its accumulated cost. To process coordinate $j$, $K$-CoreQ expands the $K$ retained states over all candidates $q_j \in \calQ_j$, scoring each extension according to:
\begin{equation}
  S_{j+1}({q}_{j+1:n}) := \sum_{t=j+1}^n \Delta_t(q_t \mid {q}_{t+1:n}), \quad S_j({q}_{j:n}) = S_{j+1} + \Delta_j(q_j \mid {q}_{j+1:n}).
  \vspace{-0.03 in}
\end{equation}
A \texttt{top-k} operation then retains the $K$ lowest-cost partial paths to serve as the beam states for coordinate $j-1$. Keeping multiple partial paths allows $K$-CoreQ to postpone making irreversible choices and reduce the effect of early rounding errors. Since the objective decomposes exactly across rows, the beam states can be tensorized, allowing the level-wise expansion and \texttt{top-k} pruning steps to run on GPU across rows. Details of the batched implementation are in Appendix~\ref{app:beam-implement}.


%% file: files/framework.tex
\begin{figure}
    \includegraphics[width=1.02\linewidth]{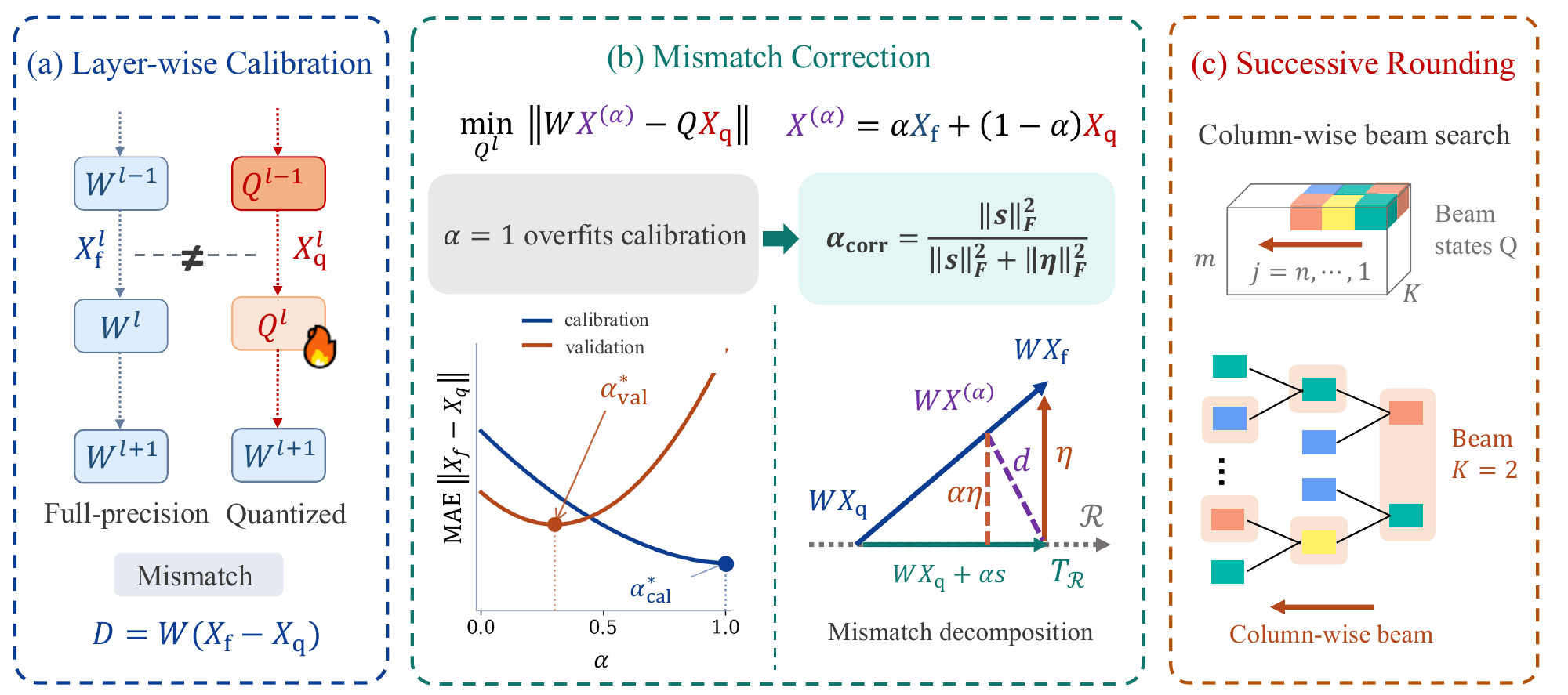}
\caption{
Overview of CoreQ.
(a) Sequential layer-wise quantization creates the mismatch between the full-precision and partially quantized activations. 
(b) CoreQ computes $\alpha_{\mathrm{corr}}$ to scale the mismatch correction according to the component reproducible by changing the current layer's weights.
(c) The resulting corrected target induces a triangular least-squares rounding problem, efficiently solved by CoreQ or $K$-CoreQ.
}
    \label{fig:framework1}
\end{figure}

%% file: files/alpha_sweep.tex

\begin{figure}
    \centering
    \begin{subfigure}[t]{1.0\linewidth}
        \centering
        \includegraphics[width=\linewidth]{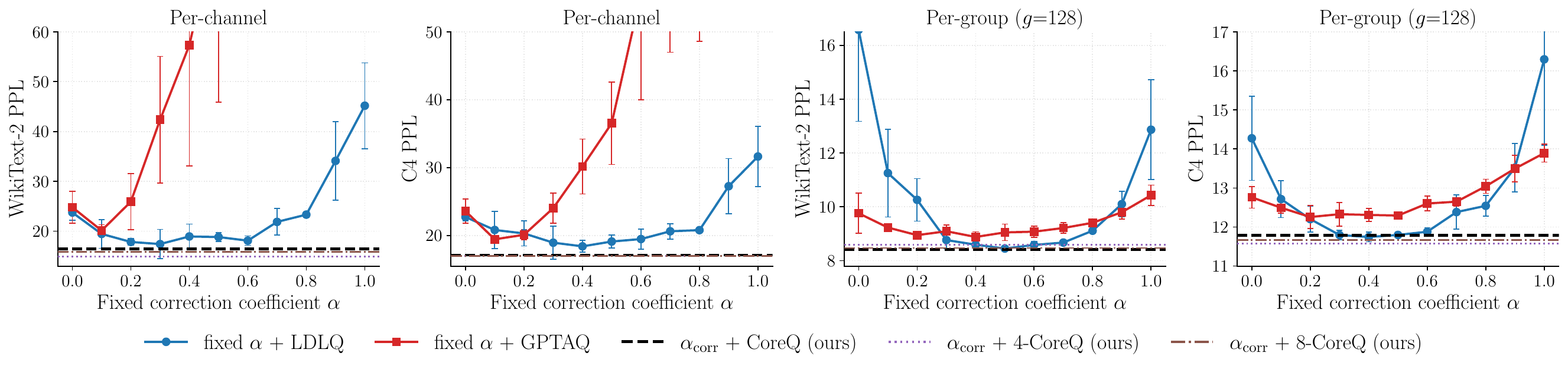}
        \vspace{-0.2 in}
        \label{fig:alpha_sweep}
    \end{subfigure}
    \begin{subfigure}[t]{1.0\linewidth}
        \centering
        \includegraphics[width=\linewidth]{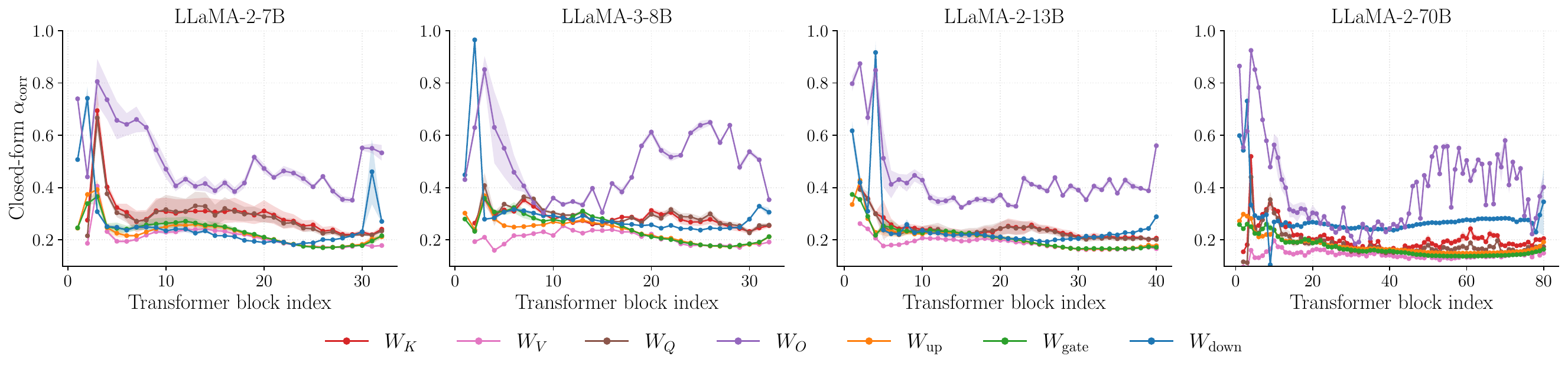}
        \vspace{-0.2 in}
        \label{fig:alpha_sweep}
    \end{subfigure}
    \caption{
    (Top) Fixed-$\alpha$ baselines use globally chosen correction coefficient on LLaMA-3-8B. CoreQ uses a per-layer $\alpha_{\mathrm{corr}}$ and matches or improves upon the best fixed $\alpha$ across granularities and datasets. (Bottom) $\alpha_{\mathrm{corr}}$ varies substantially across layer types and depths, with similar patterns across LLaMA model variants. This supports using layer-dependent correction coefficients. Additional observations are in Appendix~\ref{app:alpha-analysis}.
}
    \label{fig:alpha_analysis}
\end{figure}

%% file: main/result.tex
\input{files/tab2}

\section{Results}
\label{sec:results}
We first evaluate CoreQ in the main PTQ settings in \cref{subsec:exp_main} across model families, scales, bit-widths, and quantization granularities. We then evaluate $K$-CoreQ in \cref{subsec:exp-beam}, showing how bounded beam search improves rounding quality while controlling runtime and memory. Additional results are provided in Appendix~\ref{app:ablation}.

\vspace{-0.05 in}

\subsection{Setup}
\label{subsec:exp_setup}

We evaluate \textbf{CoreQ} for post-training quantization of decoder-only LLMs, including LLaMA~2/3~\cite{touvron2023llama, grattafiori2024llama}, Qwen3-8B~\cite{yang2025qwen3}, and Phi-3 Mini (3.8B)~\cite{abdin2024phi3}. Calibration uses 128 C4 samples with sequence length 2048. We report perplexity on WikiText2 and C4, and zero-shot accuracy on six commonsense reasoning tasks: WinoGrande~\cite{sakaguchi2021winogrande}, PIQA~\cite{bisk2020piqa}, HellaSwag~\cite{zellers2019hellaswag}, BoolQ~\cite{clark2019boolq}, ARC-Easy, and ARC-Challenge~\cite{clark2018think}. We also evaluate CoreQ on the instruction-tuned LLaMA-3.1-8B-Instruct~\cite{grattafiori2024llama} using MMLU~\cite{hendrycks2020mmlu} (5-shot) and GSM8K~\cite{cobbe2021gsm8k} with chain-of-thought prompting (8-shot); these results are reported in \cref{tab:l31-instruct} in the Appendix. Results are averaged over 5 seeds. Quantization time is reported as end-to-end wall-clock runtime on a single H200 GPU, including non-GPU overhead. We compare mainly against learning-free PTQ baselines: GPTQ~\cite{frantar2022gptq}, GPTAQ~\cite{ligptaq}, and LDLQ~\cite{chee2023quip}. Depending on the setting, we also include AWQ~\cite{lin2024awq}, OmniQuant~\cite{shao2024omniquant}, and GuidedQuant~\cite{kimguidedquant}.

\vspace{-0.05 in}
\subsection{Performance Evaluation}
\label{subsec:exp_main}

\paragraph{Weight-only quantization.}
\cref{tab:combined-3bit} reports 3-bit weight-only quantization (WOQ) under
per-channel and per-group ($g{=}128$) granularity, and the corresponding wall-clock quantization
time. On LLaMA-3-8B with per-channel quantization, CoreQ reduces Wiki2
perplexity from 24.65 for the best baseline (LDLQ) to 16.48, \textbf{a 33\% reduction}; 4-CoreQ further reduces it to 14.90. The results show two trends. First, the benefits of CoreQ grow with the strictness of the quantization constraint. With per-group quantization, where scales are updated every 128 weights, all methods remain stable but the CoreQ variants still achieve the lowest perplexity. With per-channel quantization, GPTQ and LDLQ both exceed 24 perplexity on LLaMA-3-8B, whereas CoreQ and its beam-search variants remain substantially lower. Second, bounded beam-search variants (4-CoreQ, 8-CoreQ) improve over greedy CoreQ, especially in the per-channel setting, while adding moderate runtime overhead and remaining faster than gradient-based alternatives such as GuidedQuant. 
\vspace{-0.05 in}

\input{files/tab-diff-large}

\paragraph{Evaluation on different architectures and scales.}
Beyond medium-scale LLaMA models, \cref{tab:per-channel-qwen-phi,tab:per-channel-l2-70b} show that CoreQ remains effective on Qwen3-8B, Phi-3.8B, and LLaMA-2-70B. On Qwen3-8B with per-channel quantization, CoreQ reduces Wiki2 perplexity by \textbf{12.4\%} relative to GPTAQ and achieves the lowest C4 perplexity among the compared methods. The same trend holds for Phi-3.8B, where CoreQ achieves the lowest Wiki2 and C4 perplexities, including an \textbf{8.7\%} Wiki2 reduction over GPTAQ. The benefits are most evident in the extreme low-bit regime. On 2-bit per-channel LLaMA-2-70B, where baseline Wiki2 perplexities exceed 180, CoreQ and \(K\)-CoreQ reduce Wiki2 perplexity by \textbf{48--56\%} relative to the best baseline. 
At 3 bits, CoreQ also achieves the lowest perplexity on both datasets. Overall, the gains are not limited to a single architecture family or model scale, and are especially pronounced in low-bit settings.

\input{files/tab4}

\paragraph{Weight-activation quantization.}
\label{subsec:exp_preprocess}
We also evaluate the use of CoreQ for weight refinement in weight--activation--KV cache quantization pipelines. We consider two rotation-based preprocessors: QuaRot~\cite{ashkboos2024quarot}, which applies fixed Hadamard rotations, and SpinQuant~\cite{liu2024spinquant}, which learns rotation matrices to further reduce outliers. For this experiment, calibration uses 128 WikiText2 samples. We apply per-channel symmetric quantization to weights and per-group asymmetric quantization to activations and KV caches. As shown in \cref{tab:spinquant}, CoreQ improves perplexity over learning-free baselines under both preprocessors at comparable runtimes. OmniQuant, which learns quantization parameters, requires longer quantization time and achieves higher perplexity in this setting. 
These results show that CoreQ remains effective in pipelines that also quantize activations and KV caches. However, its gains over the baselines are more modest than in the weight-only experiments. This is consistent with a geometric limit: current-layer weight updates can correct only mismatches in their reachable response space, whereas activation and KV-cache quantization can leave a larger residual outside it (see Appendix~\ref{app:kv-benefit}). Additional outlier-aware weight-only results compared to AWQ and OmniQuant are provided in \cref{tab:outlier-woq}.

\paragraph{Runtime overhead.}
Most of CoreQ's additional calibration cost comes from collecting
$\mathbf X_{\mathrm f}$ and forming the cross-statistic
$\mathbf X_{\mathrm f}\mathbf X_{\mathrm q}^{\top}$.
These mismatch-statistic computations are also required by
mismatch-aware baselines such as GPTAQ. In comparison, evaluating
$\alpha_{\mathrm{corr}}$ itself accounts for only $1.1\%$ of the
profiled GPU time. \Cref{tab:runtime-breakdown} provides a five-stage
CUDA-event profile, complementing the end-to-end wall-clock times
reported in the main tables.

\input{files/beam_figure}

\subsection{Performance-cost trade-off of $K$-CoreQ}
\label{subsec:exp-beam}

Here we evaluate $K$-CoreQ, the bounded beam search extension of CoreQ. $K$-CoreQ maintains multiple partial rounding assignments, adding controlled compute and memory to improve the triangular least-squares rounding objective. Figure~\ref{fig:beam_sweep} illustrates this trade-off for 3-bit quantization by plotting perplexity against wall-clock time and peak GPU memory. 
Increasing $K$ monotonically improves the triangular rounding objective relative to greedy CoreQ in \cref{fig:beam_ablation_main} (right). In our experiments, this tighter optimization translates to lower perplexity. Larger beams improve perplexity with modest search-time overhead under batching and increased peak memory due to the expanded beam state. Overall, beam search closes a substantial fraction of the remaining quality gap, suggesting that limited lookahead can reduce errors from early greedy decisions.

\paragraph{Comparison of rounding algorithms.}
To isolate the effect of the discrete search procedure, we compare $K$-CoreQ with cyclic coordinate descent (CD) refinement \cite{behdin2023quantease,chee2023quip}, a common PTQ post-processing step. Figure~\ref{fig:search_tradeoff} plots perplexity versus search time as we increase either the beam width $K$ or the number of CD passes $p$. Increasing $K$ results in larger perplexity reductions with small runtime overhead, whereas additional CD passes increase search time to achieve smaller gains. Figure~\ref{fig:search_efficiency} summarizes this trend by reporting the mean marginal perplexity reduction per second for each incremental update. The results indicate that, in the considered settings, bounded beam search uses the search budget more effectively than cyclic coordinate refinement.

\paragraph{Ablation of calibration and rounding mechanisms.}
\input{files/fig-ablation-beam}

We isolate the two components of CoreQ by jointly sweeping the per-layer correction coefficient
$\alpha \in \{0, 0.25, 0.5, 0.75, 1, \alpha_{\mathrm{corr}}\}$ and the beam width $K \in \{1, 2, 4, 8\}$ for 3-bit per-channel quantization of LLaMA-2-7B (\cref{fig:beam_ablation_main}, left and center). These two axes correspond to the two decision stages in layer-wise PTQ: choosing the continuous calibration target and rounding the weights to the discrete grid. The ablation shows that both components matter. 
For a fixed $\alpha$, increasing $K$ generally reduces perplexity relative to greedy rounding, apart from the more variable results at $\alpha=1$.
For a fixed $K$, $\alpha_{\mathrm{corr}}$ yields lower perplexity than the tested constant coefficients, indicating that the closed-form coefficient selects a better calibration target than a uniform correction strength. The best configuration combines $\alpha_{\mathrm{corr}}$ with $K{=}8$, showing that adaptive calibration and bounded search provide complementary performance improvements.
A complementary ablation using GPTQ-based rounding is provided
in the Appendix (\cref{tab:gptq-alpha-ablation}).

%% file: files/tab2.tex
\begin{table*}[t]
\centering
\caption{3-bit symmetric WOQ for LLaMA-2/3 models. We report Wiki2/C4 perplexity ($\downarrow$), average zero-shot accuracy ($\uparrow$), and quantization time under per-channel and per-group ($g=128$) granularity, averaged over 5 random seeds. Bold entries are the best result and values within one standard deviation of the best in each setting. GuidedQ is omitted for per-group setting following the official implementation.
}
\label{tab:combined-3bit}
\vspace{0.01 in}
\resizebox{\linewidth}{!}{
\begin{tabular}{l|l||ccc|c||ccc|c}
\toprule
& & \multicolumn{4}{c||}{Per-channel} & \multicolumn{4}{c}{Per-group (g128)} \\
Model & Method & Wiki2 ($\downarrow$) & C4 ($\downarrow$) & Avg.Acc ($\uparrow$) & Time (s) & Wiki2 ($\downarrow$) & C4 ($\downarrow$) & Avg.Acc ($\uparrow$) & Time (s) \\
\hline\hline
\multirow{8}{*}{L3-8B}
& FP16    & 6.14 & 8.93 & 74.31 & & 6.14 & 8.93 & 74.31 & \\
\cline{2-10}
& GPTQ    & 24.80 $\pm$ 3.20 & 23.59 $\pm$ 1.77 & 51.68 $\pm$ 1.21 & 550.1 & 9.87 $\pm$ 0.28 & 12.74 $\pm$ 0.40 & 66.82 $\pm$ 0.65 & 547.5 \\
& GPTAQ   & 33.47 $\pm$ 12.40 & 22.00 $\pm$ 3.10 & 51.61 $\pm$ 1.80 & 701.4 & 9.07 $\pm$ 0.33 & 12.15 $\pm$ 0.13 & 65.77 $\pm$ 1.28 & 715.0 \\
& LDLQ    & 24.65 $\pm$ 4.79 & 23.16 $\pm$ 1.92 & 51.02 $\pm$ 1.69 & 457.4 & 14.99 $\pm$ 2.85 & 14.12 $\pm$ 1.47 & 65.20 $\pm$ 2.31 & 472.5 \\
& GuidedQ & 49.79 $\pm$ 10.56 & 40.03 $\pm$ 6.55 & 48.58 $\pm$ 0.96 & 2110.1 & --- & --- & --- & --- \\
\rowcolor{blue!8}\cellcolor{white}{}& \textbf{{CoreQ}} & 16.48 $\pm$ 2.04 & \textbf{17.13} $\pm$ 0.35 & 53.99 $\pm$ 1.15 & 564.1 & \textbf{8.42} $\pm$ 0.16 & {11.79} $\pm$ 0.19 & \textbf{68.92} $\pm$ 0.58 & 575.2 \\
\rowcolor{blue!8}\cellcolor{white}{}& \textbf{{4-CoreQ}} & \textbf{14.90} $\pm$ 0.28 & \textbf{16.98} $\pm$ 0.25 & \textbf{55.74} $\pm$ 1.01 & 821.5 & 8.60 $\pm$ 0.51 & \textbf{11.57} $\pm$ 0.10 & {68.11} $\pm$ 1.65 & 784.0 \\
\rowcolor{blue!8}\cellcolor{white}{}& \textbf{{8-CoreQ}} & 15.85 $\pm$ 1.29 & \textbf{16.99} $\pm$ 0.39 & \textbf{55.24} $\pm$ 2.54 & 831.6 & \textbf{8.47} $\pm$ 0.35 & \textbf{11.66} $\pm$ 0.25 & \textbf{68.70} $\pm$ 1.46 & 855.1 \\
\hline
\multirow{8}{*}{L2-7B}
& FP16    & 5.47 & 6.97 & 70.47 & & 5.47 & 6.97 & 70.47 & \\
\cline{2-10}
& GPTQ    & 9.36 $\pm$ 0.24 & 10.71 $\pm$ 0.09 & 59.41 $\pm$ 0.78 & 461.5 & 6.75 $\pm$ 0.06 & 8.26 $\pm$ 0.01 & 66.14 $\pm$ 0.34 & 467.9 \\
& GPTAQ   & 8.84 $\pm$ 0.23 & 10.04 $\pm$ 0.06 & 60.91 $\pm$ 0.52 & 578.4 & 6.57 $\pm$ 0.03 & 8.04 $\pm$ 0.01 & 66.34 $\pm$ 0.51 & 590.4 \\
& LDLQ    & 9.42 $\pm$ 0.17 & 10.65 $\pm$ 0.11 & 60.08 $\pm$ 0.54 & 401.8 & 6.64 $\pm$ 0.10 & 8.12 $\pm$ 0.05 & 66.50 $\pm$ 0.57 & 427.8 \\
& GuidedQ & 9.01 $\pm$ 0.25 & 10.26 $\pm$ 0.07 & \textbf{61.83} $\pm$ 0.48 & 1883.0 & --- & --- & --- & --- \\
\rowcolor{blue!8}\cellcolor{white}{}& \textbf{{CoreQ}} & 8.41 $\pm$ 0.21 & 9.61 $\pm$ 0.12 & 61.46 $\pm$ 0.50 & 467.8 & {6.39} $\pm$ 0.03 & {7.86} $\pm$ 0.02 & \textbf{67.16} $\pm$ 0.34 & 470.7 \\
\rowcolor{blue!8}\cellcolor{white}{}& \textbf{{4-CoreQ}} & \textbf{8.03} $\pm$ 0.14 & 9.41 $\pm$ 0.05 & \textbf{62.64} $\pm$ 0.89 & 667.6 & \textbf{6.31} $\pm$ 0.02 & \textbf{7.82} $\pm$ 0.01 & \textbf{67.12} $\pm$ 0.24 & 701.2 \\
\rowcolor{blue!8}\cellcolor{white}{}& \textbf{{8-CoreQ}} & \textbf{7.97} $\pm$ 0.07 & \textbf{9.37} $\pm$ 0.01 & \textbf{62.53} $\pm$ 0.52 & 684.6 & \textbf{6.31} $\pm$ 0.02 & \textbf{7.83} $\pm$ 0.02 & \textbf{67.05} $\pm$ 0.36 & 709.4 \\
\hline
\multirow{8}{*}{L2-13B}
& FP16    & 4.88 & 6.47 & 73.23 & & 4.88 & 6.47 & 73.23 & \\
\cline{2-10} 
& GPTQ    & 6.90 $\pm$ 0.04 & 8.36 $\pm$ 0.02 & 66.26 $\pm$ 0.27 & 763.7 & 5.60 $\pm$ 0.03 & 7.16 $\pm$ 0.01 & \textbf{70.93} $\pm$ 0.65 & 820.5 \\
& GPTAQ   & 6.70 $\pm$ 0.04 & 8.20 $\pm$ 0.02 & 66.83 $\pm$ 0.35 & 1015.8 & {5.54} $\pm$ 0.01 & 7.10 $\pm$ 0.00 & \textbf{71.00} $\pm$ 0.11 & 1012.0 \\
& LDLQ    & 6.87 $\pm$ 0.05 & 8.32 $\pm$ 0.02 & 66.97 $\pm$ 0.43 & 704.2 & 5.57 $\pm$ 0.02 & 7.12 $\pm$ 0.00 & \textbf{70.97} $\pm$ 0.36 & 719.2 \\
& GuidedQ & 6.94 $\pm$ 0.02 & 8.37 $\pm$ 0.02 & 67.18 $\pm$ 0.26 & 3170.4 & --- & --- & --- & --- \\
\rowcolor{blue!8}\cellcolor{white}{}& \textbf{{CoreQ}} & 6.78 $\pm$ 0.29 & 8.05 $\pm$ 0.10 & 67.03 $\pm$ 0.85 & 840.1 & {5.53} $\pm$ 0.04 & {7.06} $\pm$ 0.03 & \textbf{71.19} $\pm$ 0.56 & 824.3 \\
\rowcolor{blue!8}\cellcolor{white}{}& \textbf{{4-CoreQ}} & 6.72 $\pm$ 0.18 & \textbf{8.01} $\pm$ 0.10 & 67.15 $\pm$ 0.99 & 1144.0 & \textbf{5.51} $\pm$ 0.03 & \textbf{7.04} $\pm$ 0.01 & \textbf{70.87} $\pm$ 0.27 & 1189.5 \\
\rowcolor{blue!8}\cellcolor{white}{}& \textbf{{8-CoreQ}} & \textbf{6.53} $\pm$ 0.07 & \textbf{7.98} $\pm$ 0.05 & \textbf{67.86} $\pm$ 0.58 & 1199.2 & \textbf{5.50} $\pm$ 0.02 & \textbf{7.04} $\pm$ 0.01 & \textbf{71.03} $\pm$ 0.30 & 1242.0 \\
\bottomrule
\end{tabular}
}
\vspace{-0.05 in}
\end{table*}

%% file: files/tab-diff-large.tex
\begin{table}[t]
\begin{minipage}{0.5\linewidth}
\centering
\caption{3-bit per-channel WOQ on Qwen3-8B and Phi3-3.8B. Extended results are provided in \cref{tab:qwen-phi}.}
\vspace{0.05 in}
\label{tab:per-channel-qwen-phi}
\resizebox{\linewidth}{!}{
\begin{tabular}{l|l||cc|c}
\toprule
Model & Method & Wiki2 ($\downarrow$) & C4 ($\downarrow$) & Q.Time (s) \\
\hline\hline
\multirow{5}{*}{Qwen3-8B} & FP16    & 9.73 & 14.65 & -- \\
\cline{2-5}
 & GPTQ    & 18.75 $\pm$ 0.45 & 19.69 $\pm$ 0.12 & 499.2 \\
 & GPTAQ   & 17.63 $\pm$ 0.59 & 19.08 $\pm$ 0.28 & 647.2 \\
 & LDLQ    & 16.18 $\pm$ 0.18 & 18.49 $\pm$ 0.08 & 375.4 \\
\rowcolor{blue!8}\cellcolor{white}{} & {\textbf{CoreQ}} & \textbf{15.44} $\pm$ 0.23 & \textbf{17.51} $\pm$ 0.08 & 553.9 \\
\hline
\multirow{5}{*}{Phi3-3.8B} & FP16    & 6.01 & 9.11 & -- \\
\cline{2-5}
 & GPTQ    & 15.46 $\pm$ 0.50 & 15.80 $\pm$ 0.05 & 266.5 \\
 & GPTAQ   & 12.99 $\pm$ 0.28 & 14.21 $\pm$ 0.15 & 335.7 \\
 & LDLQ    & 15.75 $\pm$ 0.74 & 16.03 $\pm$ 0.48 & 216.1 \\
\rowcolor{blue!8}\cellcolor{white}{} & {\textbf{CoreQ}} & \textbf{11.86} $\pm$ 0.34 & \textbf{13.68} $\pm$ 0.32 & 302.1 \\
\bottomrule
\end{tabular}
}
\end{minipage}
\hfill
\begin{minipage}{0.46\linewidth}
\centering
\caption{2/3-bit per-channel WOQ on L2-70B. Extended results are provided in \cref{tab:l2-70b}.}
\vspace{0.05 in}
\label{tab:per-channel-l2-70b}
\resizebox{\linewidth}{!}{
\begin{tabular}{c|l||cc|c}
\toprule
Bits & Method & Wiki2 ($\downarrow$) & C4 ($\downarrow$) & Q.Time (s) \\
\hline\hline
16 & FP16 & 3.32 & 5.52 & -- \\
\hline
\multirow{5}{*}{2}
 & GPTQ   & 181.31 $\pm$ 19.81 & 83.17 $\pm$ 5.14 & 3301.2 \\
 & GPTAQ  & 251.82 $\pm$ 80.72 & 83.59 $\pm$ 11.38 & 4663.4 \\
 & LDLQ   & 193.51 $\pm$ 36.21 & 91.32 $\pm$ 10.08 & 2753.5 \\
\rowcolor{blue!8}\cellcolor{white}{} & \textbf{CoreQ} & 94.97 $\pm$ 12.49 & \textbf{38.46} $\pm$ 2.37 & 4731.9 \\
\rowcolor{blue!8}\cellcolor{white}{}  & \textbf{4-CoreQ} & \textbf{91.68} $\pm$ 11.59 & \textbf{39.20} $\pm$ 4.62 & 6340.9 \\
\rowcolor{blue!8}\cellcolor{white}{}  & \textbf{8-CoreQ} & \textbf{79.07} $\pm$ 14.91 & \textbf{37.07} $\pm$ 2.54 & 7973.7 \\
\hline
\multirow{5}{*}{3}
 & GPTQ   & 5.11 $\pm$ 0.02 & 6.77 $\pm$ 0.02 & 3256.0 \\
 & GPTAQ  & 5.09 $\pm$ 0.05 & 6.73 $\pm$ 0.01 & 4727.5  \\
 & LDLQ   & 5.06 $\pm$ 0.02 & 6.73 $\pm$ 0.02 & 2855.0  \\
\rowcolor{blue!8}\cellcolor{white}{}  & \textbf{CoreQ} & \textbf{4.83} $\pm$ 0.03 & {6.52} $\pm$ 0.01 & 4785.1 \\
\rowcolor{blue!8}\cellcolor{white}{}  & \textbf{4-CoreQ} & \textbf{4.83} $\pm$ 0.02 & 6.52 $\pm$ 0.01 & 6264.6 \\
\rowcolor{blue!8}\cellcolor{white}{}  & \textbf{8-CoreQ} & \textbf{4.82} $\pm$ 0.01 & \textbf{6.51} $\pm$ 0.00 & 8020.5 \\
\bottomrule
\end{tabular}
}
\end{minipage}
\vspace{-0.2 in}
\end{table}

%% file: files/tab4.tex
\begin{table}[t]
\centering
\caption{Weight--activation--KV-cache quantization. We report Wiki2 perplexity and quantization time for W4A4KV4 and W4A4KV16. Best results and values within one standard deviation of the best are highlighted. OmniQuant is omitted for L3-8B due to perplexity above $100$.}
\vspace{0.05 in}
\label{tab:spinquant}
\setlength{\tabcolsep}{3pt}
\resizebox{\linewidth}{!}{
    \begin{tabular}{l||cc|cc|cc||cc|cc|cc}
    \toprule
    & \multicolumn{6}{c||}{\textbf{W4A4KV4}}
    & \multicolumn{6}{c}{\textbf{W4A4KV16}} \\
    \multirow{2}{*}{Methods}
        & \multicolumn{2}{c|}{L3-8B} & \multicolumn{2}{c|}{L2-7B} & \multicolumn{2}{c||}{L2-13B}
        & \multicolumn{2}{c|}{L3-8B} & \multicolumn{2}{c|}{L2-7B} & \multicolumn{2}{c}{L2-13B} \\
      & Wiki2 & T (s) & Wiki2 & T (s) & Wiki2 & T (s)
      & Wiki2 & T (s) & Wiki2 & T (s) & Wiki2 & T (s) \\
    \hline\hline
    OmniQuant
      & -- & -- & 15.12 $\pm$ 0.21 & 3848 & 11.82 $\pm$ 0.06 & 6654
      & -- & -- & 12.21 $\pm$ 0.09 & 3640 & 9.95 $\pm$ 0.12 & 6337 \\
    \hline
    QuaRot+GPTQ
      & 8.83 $\pm$ 0.19 & 811 & 6.21 $\pm$ 0.02 & 1445 & 5.48 $\pm$ 0.02 & 2131
      & 7.86 $\pm$ 0.06 & 793 & 6.08 $\pm$ 0.02 & 1317 & 5.35 $\pm$ 0.01 & 1977 \\
    QuaRot+GPTAQ
      & 8.08 $\pm$ 0.02 & 954 & 5.97 $\pm$ 0.00 & 710  & 5.30 $\pm$ 0.01 & 1725
      & 7.37 $\pm$ 0.02 & 918 & 5.87 $\pm$ 0.01 & 720  & 5.17 $\pm$ 0.01 & 2157 \\
    \rowcolor{blue!8} \textbf{QuaRot+{CoreQ}}
      & \textbf{7.96} $\pm$ 0.02 & 892 & \textbf{5.93} $\pm$ 0.01 & 1420 & \textbf{5.26} $\pm$ 0.00 & 2050
      & \textbf{7.20} $\pm$ 0.01 & 885 & \textbf{5.82} $\pm$ 0.00 & 1366 & \textbf{5.14} $\pm$ 0.00 & 2068 \\
    \hline
    SpinQuant+GPTQ
      & 7.37 $\pm$ 0.01 & 577 & 5.95 $\pm$ 0.01 & 521 & 5.24 $\pm$ 0.01 & 878
      & 7.28 $\pm$ 0.01 & 573 & 5.89 $\pm$ 0.01 & 523 & 5.21 $\pm$ 0.01 & 829 \\
    SpinQuant+GPTAQ
      & 7.30 $\pm$ 0.00 & 844 & 5.90 $\pm$ 0.01 & 753 & 5.22 $\pm$ 0.00 & 1121
      & 7.20 $\pm$ 0.01 & 800 & 5.85 $\pm$ 0.00 & 710 & 5.19 $\pm$ 0.01 & 1123 \\
    \rowcolor{blue!8} \textbf{SpinQuant+{CoreQ}}
      & \textbf{7.26} $\pm$ 0.01 & 683 & \textbf{5.88} $\pm$ 0.00 & 574 & \textbf{5.21} $\pm$ 0.00 & 1151
      & \textbf{7.16} $\pm$ 0.01 & 690 & \textbf{5.82} $\pm$ 0.01 & 576 & \textbf{5.15} $\pm$ 0.00 & 995 \\
    \bottomrule
    \end{tabular}
}
\vspace{-0.1 in}
\end{table}

%% file: files/beam_figure.tex
\begin{figure}[t]
    \centering
    \begin{subfigure}[c]{0.67\linewidth}
        \centering
        \includegraphics[width=1.02\linewidth]{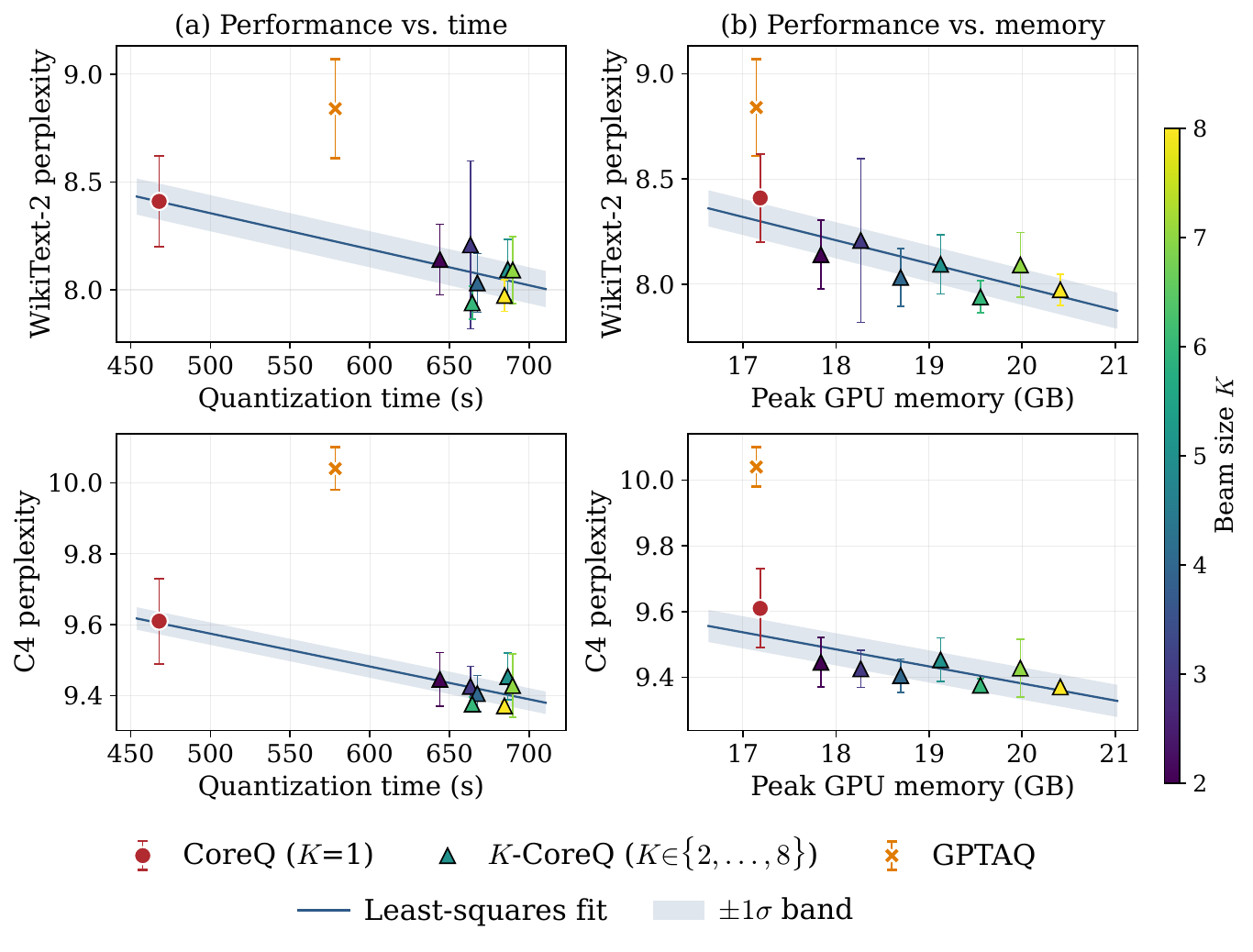}
        \vspace{-0.27 in}
        \caption{}
        \label{fig:beam_sweep}
    \end{subfigure}
    \hfill
    \begin{subfigure}[c]{0.3\linewidth}
        \centering
        \includegraphics[width=\linewidth]{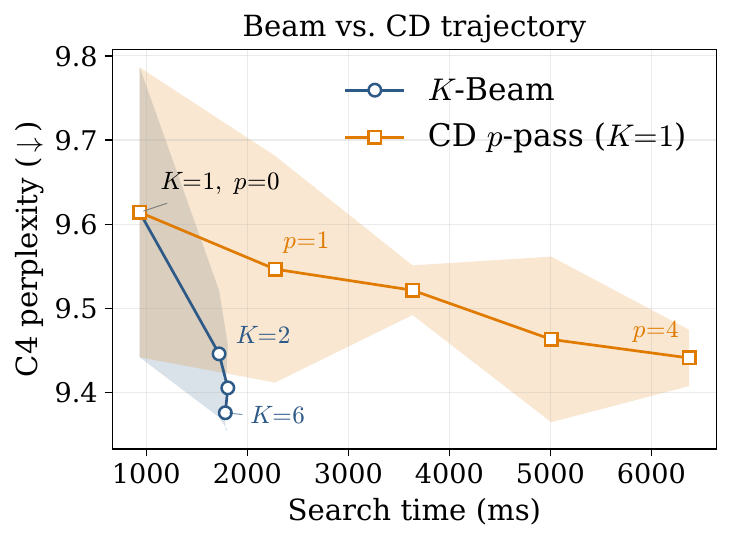}
        \vspace{-0.16 in}
        \caption{}
        \label{fig:search_tradeoff}
        \includegraphics[width=\linewidth]{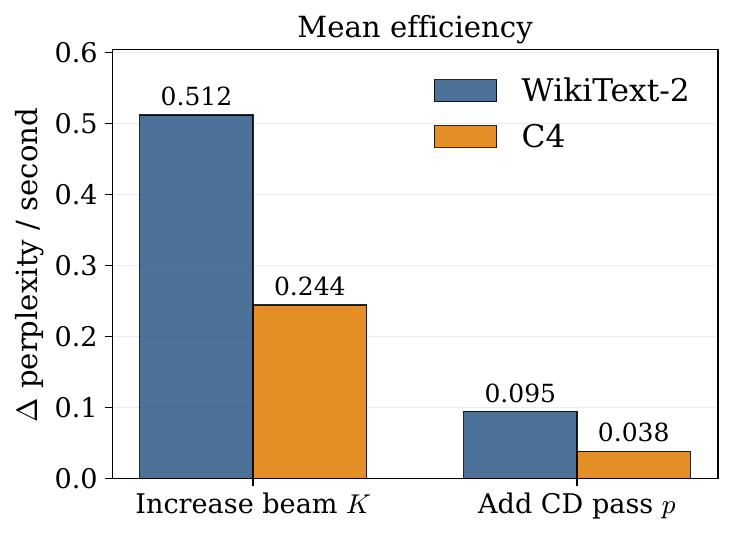}
        \vspace{-0.15 in}
        \caption{}
        \label{fig:search_efficiency}
    \end{subfigure}
    \vspace{-0.1 in}
    \caption{(a) Performance--efficiency trade-offs of $K$-CoreQ for 3-bit L2-7B. WikiText-2 and C4 perplexity are plotted vs. quantization time (left column) and peak GPU memory (right column) as beam width $K$ increases. (b) C4 perplexity vs. search time for increasing beam width $K$ and additional coordinate-descent (CD) passes ($\pm 1$ std. shaded). (c) Mean marginal perplexity improvement per second: beam expansion vs. CD passes.}
    \label{fig:beam_and_rounding}
    \vspace{-0.08 in}
\end{figure}

%% file: files/fig-ablation-beam.tex
\begin{figure}[t]
    \centering
    \begin{subfigure}[c]{0.32\linewidth}
        \centering
        \includegraphics[width=1.00\linewidth]{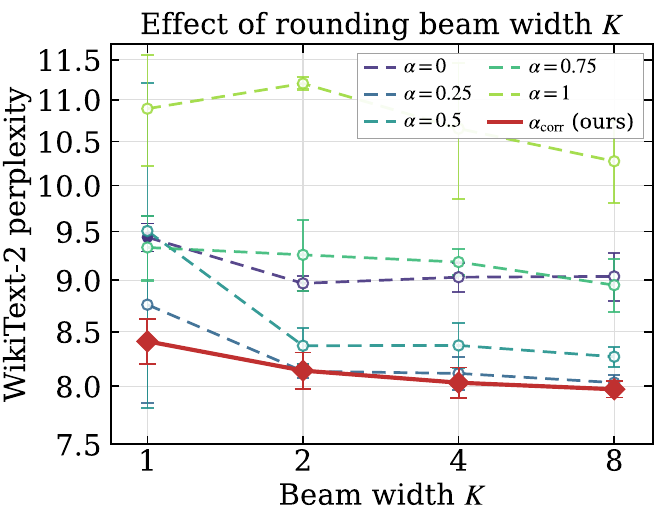}
    \end{subfigure}
    \hfill
    \begin{subfigure}[c]{0.32\linewidth}
        \centering
        \includegraphics[width=1.00\linewidth]{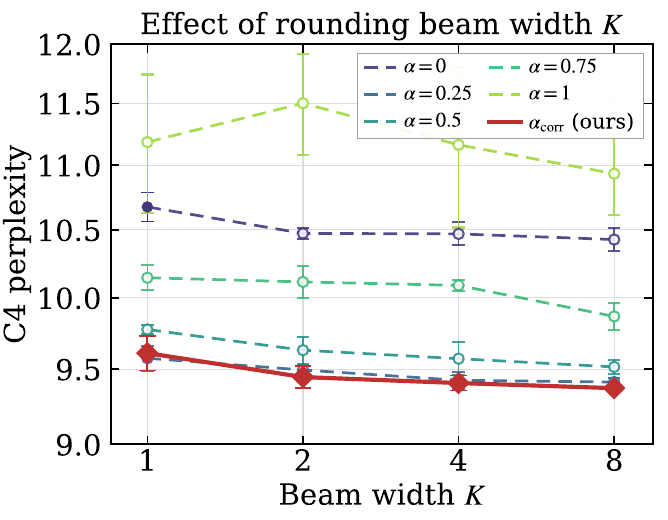}
    \end{subfigure}
    \hfill
    \begin{subfigure}[c]{0.32\linewidth}
        \centering
        \includegraphics[width=1.00\linewidth]{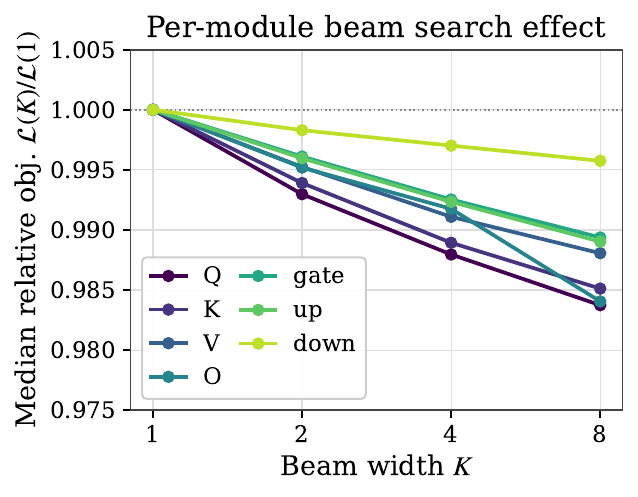}
    \end{subfigure}
    \caption{Effect of beam width $K$ on quantization quality for 3-bit per-channel LLaMA-2-7B.
    \textbf{Left, middle:} WikiText-2 and C4 perplexity across correction strengths $\alpha \in \{0,0.25,0.5,0.75,1\}$ and $\alpha_{\mathrm{corr}}$. $\alpha_{\mathrm{corr}}$ outperforms every fixed
    $\alpha$ at every $K$, while increasing $K$ reduces perplexity.
    \textbf{Right:} reduction of the triangular rounding objective under $\alpha=0$.}
    \label{fig:beam_ablation_main}
    \vspace{-0.15 in}
\end{figure}

%% file: main/related_work.tex
\section{Related Work}
\label{sec:related}

\subsection{Layer-wise Calibration Objectives for PTQ}
\label{subsec:rw-objective}
A widely used family of PTQ methods minimizes a Hessian-based calibration objective, often called a layer-wise reconstruction error
\cite{nagel2020up, hubara2021accurate, li2021brecq, frantar2022obq,
frantar2022gptq}. GPTQ is the standard LLM-scale realization of this approach. Several extensions incorporate end-to-end information through end-loss
gradients, KL-based Hessian approximations, or output cross-entropy
distortion \cite{kimguidedquant, tseng2025model, edalati2025oac}, but at substantially higher computational cost. More recent methods instead address upstream quantization errors by redefining the layer-wise calibration objective. GPTAQ \cite{ligptaq} and Qronos~\cite{zhang2026qronos, zhang2025provable} target the observed mismatch by matching the full-precision layer output on full-precision activations to the quantized layer output on activations generated by the quantized prefix. 
Similarly, QEP \cite{arai2026qep} introduces a manually tuned coefficient that interpolates between standard Hessian-based objective and full mismatch-aware calibration. Other methods address error propagation via learning-based adapters \cite{liao2024apiq} or joint block quantization \cite{ding2023cbq}. Unlike methods that apply full correction ($\alpha = 1$) or use a fixed
global $\alpha$, CoreQ computes a closed-form per-layer correction coefficient $\alpha_{\mathrm{corr}}$ from calibration statistics alone. Learning-based methods such as OmniQuant and LRQ optimize auxiliary parameters by gradient descent
\cite{shao2024omniquant, lee2025lrq}, trading additional compute for accuracy. An extended discussion is provided in Appendix~\ref{app:extended-related-work}.

\subsection{Discrete Rounding and Search Algorithms}
\label{subsec:rw-rounding}
Given a Hessian-based quadratic objective, PTQ reduces to discrete optimization over low-bit weights. GPTQ and its equivalent LDLQ realization are widely used scalar, column-wise greedy rounding solvers \cite{frantar2022gptq, chee2023quip}. Subsequent work improves rounding quality in two main ways: iterative refinement with coordinate-descent passes which trade additional objective evaluations for lower quantization error \cite{behdin2023quantease, nair2024cdquant}, and refined error compensation which accounts for first-order effects of quantization error during column-wise updates \cite{zheng2025foem}. A complementary geometric view interprets GPTQ as Babai's nearest-plane method for the closest vector problem, motivating lookahead-based search beyond the greedy path \cite{chen2025geometry, birnick2026lattice}. A separate family of vector quantization methods, including AQLM, QuIP\#, and GPTVQ, shifts the discrete problem to codebook decoding and applies beam search over codeword assignments \cite{egiazarian2024aqlm, tseng2024quip, van2024gptvq}. Our beam-search extension, $K$-CoreQ, instead performs search over the successive column-wise rounding decisions induced by the triangular least-squares proxy. Rather than committing to one greedy path, as done by Babai's algorithm, $K$-CoreQ maintains $K$ candidate paths through a bounded tree search, leading to a better accuracy-runtime trade-off than coordinate-descent refinement within the same Hessian-based objective. 

%% file: main/conclusion.tex
\section{Conclusion}

We presented CoreQ, a learning-free PTQ framework that improves layer-wise quantization under limited calibration data. CoreQ selects a corrected calibration target using a closed-form, layer-dependent correction coefficient, avoiding gradient updates, validation search, and tuning. This makes CoreQ a practical drop-in alternative to standard Hessian-based PTQ solvers: CoreQ serves as an efficient default, while \(K\)-CoreQ provides a search-enhanced option when additional calibration compute is available. Experiments show that adaptive mismatch correction consistently improves low-bit quantization, suggesting CoreQ as a strong default strategy for efficient LLM deployment.

\textbf{Limitations and Future Work.}
While CoreQ improves over prior layer-wise PTQ methods, it remains sequential and layer-wise: the corrected calibration target accounts for errors from previously quantized layers, but CoreQ does not jointly optimize multiple layers. CoreQ also restricts calibration target selection to linear interpolation between the standard and mismatch-aware calibration targets. Developing efficient nonlinear calibration target corrections for more complex mismatches is an interesting direction for future work.

%% file: appendix/appendix.tex
\input{appendix/coreq_appendix_theory}

\section{Rounding Algorithms}

\paragraph{Notation.}
In this appendix, bold uppercase symbols denote matrices. In particular,
\(\mathbf{W}\in\mathbb{R}^{m\times n}\) denotes the full-precision weight,
\(\mathbf{W}_{\mathrm{corr}} := \mathbf{W} + \alpha_{\mathrm{corr}}\widetilde{\mathbf{S}}\)
denotes the mismatch-corrected continuous weight (or rounding center),
\(\mathbf{Q}\in\mathbb{R}^{m\times n}\) denotes the quantized weight,
\(\mathbf{H}\in\mathbb{R}^{n\times n}\) denotes the Hessian proxy, and
\(\mathbf{L}\) its Cholesky factor. We use
\(\mathbf{Q}_{:,j}\) for column \(j\), \(\mathbf{Q}_{i,:}\) for row \(i\),
and keep scalar coordinates such as \(q_j\), \(c_j\), and \(\Delta_j\)
non-bold.

We provide pseudocode for CoreQ and \(K\)-CoreQ, together with practical implementation details that bridge the objective reformulation and search procedures used in our codebase. Specifically, our algorithm for computing \(\alpha_{\mathrm{corr}}\) is provided in \cref{alg:alpha-corr}. \cref{alg:coreq-lazy-batch} presents the CoreQ algorithm with lazy-batch implementation \cite{frantar2022gptq}. Building on the same triangular proxy, \cref{alg:kcoreq-lazy-batch} extends CoreQ to a \(K\)-best beam search that maintains multiple candidates per output row and enables a bounded-complexity improvement over greedy rounding.

To make the connection explicit, we first provide a matrix-level interpretation that shows how the Cholesky-triangular form yields a columnwise branch-metric decomposition with an interference-cancelled center (\cref{prop:sphere_obj_decomp}), which directly motivates successive rounding and its beam-search generalization. We then detail our beam-search implementation.

\subsection{Matrix-level interpretation of objective reformulation and successive rounding}
\label{app:rounding-detail}

\begin{proposition}[Columnwise decomposition of the triangular proxy]
\label{prop:sphere_obj_decomp}
Let \(\mathbf{H}\in\mathbb{R}^{n\times n}\) be SPD and admit a Cholesky factorization
\(\mathbf{H}=\mathbf{L}\mathbf{L}^\top\), where \(\mathbf{L}\) is lower triangular and
\(\mathbf{L}_{j,j}>0\). Define
\(\mathbf{E}:=\mathbf{Q}-\mathbf{W}_{\mathrm{corr}}\in\mathbb{R}^{m\times n}\).
Given the calibration objective
\[
    \mathcal{L}(\mathbf{Q})
    =
    \|(\mathbf{Q}-\mathbf{W}_{\mathrm{corr}})\mathbf{L}\|_F^2
    =
    \|\mathbf{E}\mathbf{L}\|_F^2,
\]
we have
\begin{equation}
\mathcal{L}(\mathbf{Q})
=
\sum_{j=1}^n
\mathbf{L}_{j,j}^2
\left\|
\mathbf{E}_{:,j}
+
\sum_{k=j+1}^n
\mathbf{E}_{:,k}
\frac{\mathbf{L}_{k,j}}{\mathbf{L}_{j,j}}
\right\|_2^2 .
\label{eq:tri_metric_columnwise}
\end{equation}
Equivalently, defining
\begin{equation}
\widetilde{\mathbf{W}}_{:,j}
:=
\mathbf{W}_{\mathrm{corr},\,:,j}
+
\sum_{k=j+1}^n
(\mathbf{W}_{\mathrm{corr},\,:,k}-\mathbf{Q}_{:,k})
\frac{\mathbf{L}_{k,j}}{\mathbf{L}_{j,j}},
\label{eq:tildeW_def}
\end{equation}
we have the diagonalized form
\begin{equation}
\mathcal{L}(\mathbf{Q})
=
\sum_{j=1}^n
\mathbf{L}_{j,j}^2
\,
\|\widetilde{\mathbf{W}}_{:,j}-\mathbf{Q}_{:,j}\|_2^2 .
\label{eq:sphere_obj}
\end{equation}
\end{proposition}
\cref{eq:sphere_obj} is precisely the branch-metric decomposition that appears in
sphere decoding: \(\mathbf L_{j,j}^2\) acts as a per-level weight, and
\(\widetilde{\mathbf W}_{:,j}\) is the interference-cancelled center for column \(j\)
given already-chosen later columns \(\{\mathbf Q_{:,k}\}_{k>j}\). Thus, choosing
\(\mathbf Q\) corresponds to \emph{traversing a tree} over columns
\(j=n,n-1,\dots,1\), where each decision fixes a column and updates the targets
for earlier columns.

\begin{proof}
Let \(\mathbf e_j\) denote the \(j\)-th standard basis vector in \(\mathbb R^n\).
The \(j\)-th column of \(\mathbf E\mathbf L\) is
\[
(\mathbf E\mathbf L)_{:,j}
=
\mathbf E\mathbf L_{:,j}
=
\sum_{k=1}^n \mathbf E_{:,k}\mathbf L_{k,j}.
\]
Since \(\mathbf L\) is lower triangular, \(\mathbf L_{k,j}=0\) for \(k<j\), hence
\[
(\mathbf E\mathbf L)_{:,j}
=
\sum_{k=j}^n \mathbf E_{:,k}\mathbf L_{k,j}
=
\mathbf L_{j,j}
\left(
\mathbf E_{:,j}
+
\sum_{k=j+1}^n
\mathbf E_{:,k}
\frac{\mathbf L_{k,j}}{\mathbf L_{j,j}}
\right).
\]
Taking squared \(\ell_2\) norms and summing over \(j\) yields
\begin{equation}
\|\mathbf E\mathbf L\|_F^2
=
\sum_{j=1}^n
\|(\mathbf E\mathbf L)_{:,j}\|_2^2
=
\sum_{j=1}^n
\mathbf L_{j,j}^2
\left\|
\mathbf E_{:,j}
+
\sum_{k=j+1}^n
\mathbf E_{:,k}
\frac{\mathbf L_{k,j}}{\mathbf L_{j,j}}
\right\|_2^2,
\end{equation}
which proves \eqref{eq:tri_metric_columnwise}. Substituting
\(\mathbf{E}=\mathbf{Q}-\mathbf{W}_{\mathrm{corr}}\) and rearranging gives
\[
\widetilde{\mathbf{W}}_{:,j}-\mathbf{Q}_{:,j}
=
-
\left(
\mathbf{E}_{:,j}
+
\sum_{k>j}
\mathbf{E}_{:,k}
\frac{\mathbf{L}_{k,j}}{\mathbf{L}_{j,j}}
\right),
\]
and thus \eqref{eq:sphere_obj}.
\end{proof}

\paragraph{Permutation matrix \(\mathbf P\).}

The reformulation in \eqref{eq:sphere_obj} also motivates a column permutation of the weight matrix prior to successive rounding. In \eqref{eq:sphere_obj}, the proxy decomposes into level-wise contributions
\(\mathbf L_{j,j}^2\|\widetilde{\mathbf W}_{:,j}-\mathbf Q_{:,j}\|_2^2\), where the weight \(\mathbf L_{j,j}^2\) reflects the local scaling induced by \(\mathbf H\). Since \(\widetilde{\mathbf W}_{:,j}\) is an interference-cancelled center that depends on the already-fixed columns \(\{\mathbf Q_{:,k}\}_{k>j}\), the overall procedure is order-dependent: decisions made first establish the tail \(\{\mathbf Q_{:,k}\}_{k>j}\) that conditions all subsequent columns.

Accordingly, we define a permutation \(\pi:[n]\rightarrow[n]\) that sorts coordinates as
\[
\mathbf H_{\pi(1),\pi(1)}
\le
\mathbf H_{\pi(2),\pi(2)}
\le
\cdots
\le
\mathbf H_{\pi(n),\pi(n)}.
\]
We then apply this permutation to reorder the columns before running successive rounding with a right-to-left traversal \(j=n,n-1,\dots,1\). Under this convention, the coordinates with larger diagonal curvature \(\mathbf H_{ii}\) are placed closer to the beginning of the decoding order and are therefore quantized first. This permutation is used throughout all of our implementations.

We note that this interpretation is exact only in the diagonal case: if \(\mathbf H\) is diagonal, then its Cholesky factor satisfies
\(
\mathbf L
=
\operatorname{diag}
\left(
\sqrt{\mathbf H_{11}},\ldots,\sqrt{\mathbf H_{nn}}
\right),
\)
so sorting by \(\mathbf H_{ii}\) is equivalent to sorting by the true per-level weights
\(\mathbf L_{j,j}^2\) in \eqref{eq:sphere_obj}. When \(\mathbf H\) is not diagonal, the diagonal entries \(\mathbf H_{ii}\) no longer uniquely determine \(\mathbf L_{j,j}\) due to off-diagonal couplings. Nevertheless, in practice we find that ordering columns by the diagonal proxy \(\mathbf H_{ii}\) yields a consistent and effective decoding order, serving as a simple curvature-based heuristic that remains well aligned with the dominant directions even in the presence of correlations.

\subsection{Beam Search Implementation}
\label{app:beam-implement}

At a given level, or column, \(j=n,\dots,1\), for each row and beam, we form the interference-cancelled center using the triangular couplings. In the columnwise batched form, this corresponds to the \(\widetilde{\mathbf{W}}_{:,j}\) update in \eqref{eq:tildeW_def}. In our code, we compute the same quantity as a beam-dependent tensor \(\widetilde{\mathbf{W}}\in\mathbb{R}^{m\times K}\) via batched matrix multiplications.

Given \(\widetilde{\mathbf{W}}\) for column \(j\), we evaluate candidate quantization levels \(q_j\in\mathcal{Q}_j\) by broadcasting and compute the incremental costs
\(
\boldsymbol{\Delta}_j
=
\mathbf{L}_{j,j}^2
\,
|q_j-\widetilde{\mathbf{W}}|^2
\in
\mathbb{R}^{m\times K\times A}.
\)
We then update scores
\(
\mathbf S'
=
\mathbf S+\boldsymbol{\Delta}_j,
\)
broadcast over the \(A\) candidate levels, flatten the beam-and-choice axis, and keep the best \(K\) per row via \texttt{topk}, letting us update the tail state \(\mathbf Q_{\mathrm{tail}}\). This implements an exact \(K\)-best search under the triangular branch metrics while remaining GPU-friendly. \cref{alg:kcoreq-lazy-batch} shows the \(K\)-CoreQ algorithm with lazy-batch update, and the dimensions of matrices and tensors used are presented in \cref{tab:beam_dims}.

For clarity, we define the three auxiliary operators used in \cref{alg:kcoreq-lazy-batch}. All of them operate row-wise, independently for each output row \(r\in[m]\), since the proxy objective decomposes over rows.

\smallskip
\noindent\textbf{\textsc{TopK}\((\mathbf J,K)\).}
Let \(\mathbf J\in\mathbb R^{m\times K\times A}\) be a tensor of candidate costs, where \(K\) indexes the current beam and \(A\) indexes the discrete quantization levels. For each row \(r\), define the set of all beam--level pairs
\(
\Omega=\{(b,a): b\in[K],\,a\in[A]\},
\)
and the corresponding costs
\(
\rho_r(b,a):=\mathbf J_{r,b,a}.
\)
\textsc{TopK} returns the \(K\) smallest costs and their argmin indices:
\[
\{(b^\star_{r,k},a^\star_{r,k})\}_{k=1}^K
=
\mathrm{KSmallest}
\big(
\{\rho_r(b,a):(b,a)\in\Omega\},K
\big),
\]
and outputs
\[
\mathbf S_{r,k}
\leftarrow
\mathbf J_{r,b^\star_{r,k},a^\star_{r,k}},
\qquad
\bm{parent}_{r,k}
\leftarrow
b^\star_{r,k},
\qquad
\bm{choice}_{r,k}
\leftarrow
a^\star_{r,k}.
\]
Intuitively, at each level we expand every current beam by all \(A\) levels, and keep the best \(K\) expansions per row.

\smallskip
\noindent\textbf{\textsc{Gather}\((\mathbf Z,\bm{parent})\).}
This is the beam reindexing operator: it reorders a beam-dependent state tensor so that each surviving beam \(k\) inherits the state of its selected parent beam \(\bm{parent}_{r,k}\). Formally, if \(\mathbf Z\in\mathbb R^{m\times K\times d}\), or more generally has any trailing shape \(d\), then
\[
\big(\textsc{Gather}(\mathbf Z,\bm{parent})\big)_{r,k,:}
=
\mathbf Z_{r,\bm{parent}_{r,k},:},
\qquad
\forall r\in[m],\,k\in[K].
\]
We apply this to \(\mathbf{Q}_{\mathrm{blk}}\), \(\mathbf{T}\), and the ancestry map \(\bm{\pi}\), so that their beam dimensions remain consistent with the pruned score tensor \(\mathbf{S}\).

\smallskip
\noindent\textbf{\textsc{PickLevel}\((\mathbf V,\bm{choice})\).}
Given candidate quantization values \(\mathbf V\in\mathbb R^{m\times A}\) for the current column and selected level indices \(\bm{choice}\in\mathbb R^{m\times K}\), \textsc{PickLevel} returns the chosen quantized values per row and beam:
\[
\big(\textsc{PickLevel}(\mathbf V,\bm{choice})\big)_{r,k}
=
\mathbf V_{r,\bm{choice}_{r,k}},
\qquad
\forall r\in[m],\,k\in[K].
\]
This is exactly the column update \(\mathbf Q_{\mathrm{blk}}[:,:,j]\) for the surviving beams.

\smallskip
\noindent
In implementation, \textsc{TopK} corresponds to a row-wise \texttt{topk} on the flattened beam--level axis, while \textsc{Gather} and \textsc{PickLevel} are indexed gathers along the beam and level dimensions, respectively.

\input{files/tab-matrix-dim}
\newpage
\input{files/alg_alpha}
\input{files/lazy-batch-algorithm}
\input{files/algorithm2}

\FloatBarrier
\section{Additional Results}
\label{app:ablation}

\subsection{Analysis of the correction coefficient \(\alpha\)}
\label{app:alpha-analysis}

\paragraph{Why a single \(\alpha\) does not suffice.}
A natural hypothesis is that increasing \(\alpha\), which pushes the calibration target toward the full-precision activation, should monotonically reduce
the layerwise mean activation error (MAE)
\(\|\mathbf X_{\mathrm f}^{\ell}-\mathbf X_{\mathrm q}^{\ell}\|\).
\cref{fig:alpha-cal-and-val} (top row) confirms this on the calibration set:
across all three models, larger \(\alpha\) yields smaller MAE, with the effect
concentrated in the deeper layers where upstream errors have accumulated. The
trend does not transfer to held-out data. On the validation set (bottom row),
\(\alpha=0.5\) matches or improves on \(\alpha=1\) at every layer for L2-7B and
L3-8B, and the gap widens with depth. On L2-13B, although \(\alpha=1\) achieves
the lowest calibration MAE in the deepest layers, it diverges from \(\alpha=0.5\) on the validation set, and the gap widens with depth. This shows that aggressive correction overfits to calibration-specific mismatch patterns
rather than recovering population-level structure. Yet a constant choice is
itself fragile. The validation MAE curves diverge across \(\alpha\) in a
layer-dependent fashion within each model (\cref{fig:alpha-cal-and-val}), so
even the best per-model constant is suboptimal at different depths.

\input{files/fig-lim-asym}

\paragraph{Empirical behavior of $\alpha_{\mathrm{corr}}$.}
The closed-form $\alpha_{\mathrm{corr}}$ is computed per layer from the calibration batch alone, with no held-out data or grid search.
\cref{fig:alpha_analysis} reports its values across four models from LLaMA-2-7B
to LLaMA-2-70B at 3-bit per-channel quantization. Two structural properties
emerge. 

\textbf{(i)~\(\alpha_{\mathrm{corr}}\) varies systematically across layer types}: for example, the output projection \(\mathbf W_O\) consistently requires the most
aggressive correction (\(\alpha_{\mathrm{corr}}\approx 0.4\)--\(0.6\) in steady
state) and the ordering persists across model family and an order-of-magnitude
scale change. One plausible explanation is that \(\mathbf W_O\), as the final
linear map in the attention branch, sees mismatch after attention aggregation and
head mixing, which often appears as a coherent feature-space shift that can be
absorbed by a linear output projection. In contrast, \(\mathbf W_V\) is applied
before the data-dependent attention mixing; its mismatch is more token- and
context-specific and therefore less well represented by a single linear
correction on \(\mathbf X_{\mathrm q}\). The fact that this ordering persists
across model families and scales suggests that the coefficient captures
layer-specific geometry of the mismatch, rather than random calibration noise.

\textbf{(ii)~\(\alpha_{\mathrm{corr}}\) varies systematically across depth.}
Early blocks transiently approach \(\alpha_{\mathrm{corr}}\approx 1\)
before settling to layer-specific steady points, reflecting the accumulation
of upstream rounding error layer by layer. This is consistent with the
reachability interpretation: near the beginning of the network, the
activation mismatch is induced by only a short quantized prefix and
is often close to a coherent first-order perturbation, making it largely
explainable by a local linear weight shift. As depth increases, upstream
rounding errors pass through attention, MLP, normalization, and residual
connections, making the mismatch more mixed and layer-dependent. The
reachable fraction therefore stabilizes to layer-specific levels rather
than remaining uniformly close to one. This behavior is desirable in PTQ:
once propagated mismatch becomes less locally reachable, full correction
would allow accumulated calibration-specific error to steer the calibration target. The depth-adaptive shrinkage instead corrects only the component
that remains explainable by the current layer and otherwise falls back
toward the standard calibration target.

\input{files/fig-alpha-sweep-llama}

\subsection{Ablation with Correction Strength and Beam Width}
We jointly vary the correction coefficient and beam width in a
GPTQ-based rounding implementation to examine their respective
contributions (\cref{tab:gptq-alpha-ablation}).
Fixed coefficients are shared across layers, whereas
$\alpha_{\mathrm{corr}}$ is computed separately for each layer.
At $K=1$, replacing the standard target ($\alpha=0$) with the
adaptive target reduces WikiText-2 perplexity from $9.36$ to $8.41$.
Increasing the beam width alone to $K=8$ yields $9.03$, while
combining adaptive correction with $K=8$ yields $7.97$.
These results demonstrate the complementary benefits of adaptive
target selection and expanded search.

\begin{table}[t]
\centering
\caption{
Ablation of correction coefficient $\alpha$ and beam width $K$
using GPTQ-based rounding on LLaMA-2-7B with $3$-bit per-channel
weight-only quantization.
Fixed-$\alpha$ settings use a global coefficient, whereas CoreQ
uses the layer-wise coefficient $\alpha_{\mathrm{corr}}$.
We report mean $\pm$ standard deviation of perplexity
(lower is better). The lowest mean in each column is bold.
}
\label{tab:gptq-alpha-ablation}
\small
\setlength{\tabcolsep}{4pt}
\begin{tabular}{lcccccc}
\toprule
& \multicolumn{3}{c}{WikiText-2}
& \multicolumn{3}{c}{C4} \\
\cmidrule(lr){2-4}
\cmidrule(lr){5-7}
Correction
& $K=1$ & $K=4$ & $K=8$
& $K=1$ & $K=4$ & $K=8$ \\
\midrule
$\alpha=0.00$
& $9.36 \pm 0.24$
& $9.09 \pm 0.11$
& $9.03 \pm 0.10$
& $10.71 \pm 0.09$
& $10.48 \pm 0.16$
& $10.39 \pm 0.02$ \\

$\alpha=0.25$
& $8.91 \pm 0.25$
& $8.57 \pm 0.04$
& $8.51 \pm 0.18$
& $10.05 \pm 0.08$
& $9.93 \pm 0.05$
& $9.85 \pm 0.02$ \\

$\alpha=0.50$
& $8.94 \pm 0.17$
& $8.67 \pm 0.12$
& $8.83 \pm 0.35$
& $10.11 \pm 0.03$
& $10.02 \pm 0.03$
& $9.90 \pm 0.01$ \\

$\alpha=0.75$
& $9.80 \pm 0.42$
& $9.19 \pm 0.04$
& $9.14 \pm 0.27$
& $10.51 \pm 0.12$
& $10.25 \pm 0.03$
& $10.19 \pm 0.03$ \\

$\alpha=1.00$
& $10.40 \pm 0.38$
& $9.90 \pm 0.17$
& $10.37 \pm 0.44$
& $11.09 \pm 0.05$
& $10.86 \pm 0.04$
& $10.76 \pm 0.05$ \\

\midrule
CoreQ ($\alpha_{\mathrm{corr}}$)
& $\mathbf{8.41} \pm 0.21$
& $\mathbf{8.03} \pm 0.14$
& $\mathbf{7.97} \pm 0.07$
& $\mathbf{9.61} \pm 0.12$
& $\mathbf{9.41} \pm 0.05$
& $\mathbf{9.37} \pm 0.01$ \\
\bottomrule
\end{tabular}
\end{table}

\subsection{Runtime Breakdown}
\label{app:runtime-breakdown}

The quantization times in the main tables are end-to-end wall-clock
measurements and include model execution, calibration-data handling,
data movement, and non-GPU overhead. To isolate the computational
stages of CoreQ, \cref{tab:runtime-breakdown} instead reports
CUDA-event GPU time for LLaMA-2-13B. The collection of
$\mathbf X_{\mathrm q}$, construction of $\mathbf H$, and greedy
rounding are shared with the corresponding LDLQ pipeline. CoreQ
additionally collects $\mathbf X_{\mathrm f}$ and forms
$\mathbf C=\mathbf X_{\mathrm f}\mathbf X_{\mathrm q}^{\top}$,
which account for most of its additional GPU cost. These computations are also required by mismatch-aware baselines such as GPTAQ. In comparison,
evaluating the closed-form coefficient accounts for only $1.1\%$ of the profiled GPU time. 

\begin{table}[!t]
\centering
\small
\caption{CUDA-event GPU-time profile of CoreQ for LLaMA-2-13B.
Percentages are relative to the sum of the five profiled stages and
exclude non-GPU wall-clock overhead.}
\vspace{0.05 in}
\label{tab:runtime-breakdown}
\setlength{\tabcolsep}{4pt}
\begin{tabular}{lrr}
\toprule
Stage & Time (s) & Share (\%) \\
\midrule
$\mathbf X_{\mathrm q}$ collection and
$\mathbf H=\mathbf X_{\mathrm q}\mathbf X_{\mathrm q}^{\top}$
& 193.2 & 28.8 \\
$\mathbf X_{\mathrm f}$ collection
& 23.1 & 3.4 \\
$\mathbf C=\mathbf X_{\mathrm f}\mathbf X_{\mathrm q}^{\top}$
& 122.1 & 18.2 \\
$\alpha_{\mathrm{corr}}$ computation
& 7.7 & 1.1 \\
Greedy successive rounding
& 324.5 & 48.4 \\
\midrule
Total profiled GPU time
& 670.6 & 100.0 \\
\bottomrule
\end{tabular}
\end{table}

\subsection{Sensitivity to Calibration Data Size}
\label{subsec:cal-size}
\input{files/tab-cal-size}

We evaluate the sensitivity of CoreQ to the number of calibration samples
used to estimate the activation statistics and mismatch correction. Table~\ref{tab:calset-sensitivity} reports results for LLaMA-2-7B under
3-bit per-channel weight-only quantization, using C4 calibration sets of
different sizes. Across all calibration sizes, CoreQ achieves lower mean
perplexity and higher average zero-shot accuracy than GPTAQ. The improvement is especially clear for C4 perplexity,
where CoreQ consistently reduces error even with only 16 or 32 calibration
samples. This suggests that the reachable/unreachable decomposition
provides a useful safeguard against fully applying the mismatch
when calibration data are limited. As the calibration set grows, both methods improve, reflecting more stable
activation statistics and mismatch estimates. CoreQ continues to provide
additional gains at 128 samples, indicating that the adaptive correction
does not merely help in extremely data-scarce regimes. 

\subsection{Additional Quantization Settings}
\paragraph{Joint weight-activation-KV cache quantization}
\label{app:kv-benefit}
To examine the smaller gains under joint quantization, we analyze
QuaRot-rotated LLaMA-2-7B with 4-bit weights under W4A16KV16,
W4A4KV16, and W4A4KV4. 
\Cref{tab:joint-geometry} reports the mean and standard deviation of
$\alpha_{\mathrm{corr}}$ across 32 layers. Its mean decreases from
$0.218$ under W4A16KV16 to $0.191$ under W4A4KV16 and $0.185$ under
W4A4KV4. Thus, as activations and
KV caches are quantized, a smaller fraction of the observed mismatch is
reachable through current-layer weight refinement. This trend is
consistent with the smaller performance gains of CoreQ under joint
quantization.

\begin{table}[h]
\centering
\caption{Correction geometry and CoreQ perplexity for QuaRot-rotated
LLaMA-2-7B with 4-bit weights. Values of
$\alpha_{\mathrm{corr}}$ are mean $\pm$ standard deviation across
32 layers.}
\vspace{0.05 in}
\label{tab:joint-geometry}
\resizebox{0.7\textwidth}{!}{%
\begin{tabular}{lccc}
\toprule
Metric
& W4A16KV16
& W4A4KV16
& W4A4KV4 \\
\midrule
$\alpha_{\mathrm{corr}}$
& $0.218 \pm 0.180$
& $0.191 \pm 0.151$
& $0.185 \pm 0.132$ \\
WikiText-2 / C4 PPL
& $5.63$ / $7.10$
& $5.82$ / $7.54$
& $5.94$ / $7.69$ \\
\bottomrule
\end{tabular}
}
\end{table}

\paragraph{Weight-only quantization with outlier reduction}

\cref{tab:outlier-woq} compares CoreQ against two outlier-aware
weight-only baselines: AWQ~\cite{lin2024awq}, which migrates outliers
from activations into weights via per-channel scaling, and
OmniQuant~\cite{shao2024omniquant}, which learns quantization
parameters via gradient descent. Combined with QuaRot~\cite{ashkboos2024quarot}
to handle activation outliers, CoreQ achieves the lowest Wiki2 perplexity, for both per-channel and per-group quantization across all three models, while running $3$--$5\times$ faster than OmniQuant. It also improves C4 perplexity over OmniQuant in every setting. AWQ achieves lower C4 perplexity on LLaMA-2-7B and LLaMA-2-13B
under per-group quantization, but requires longer quantization time than QuaRot+CoreQ for both models. On LLaMA-3-8B with per-channel quantization, QuaRot+CoreQ reduces WikiText-2 perplexity from $12.18$ for AWQ and $15.11$ for OmniQuant to $7.14$.
These results show that CoreQ composes well with outlier-reduction preprocessing and achieves strong accuracy–cost trade-offs in this setting.

\input{files/tab-outlier-red}

\paragraph{Vision transformers}
\label{app:vit-results}

To verify that CoreQ's calibration objective transfers beyond
language models, we evaluate W4A4 quantization on DeiT-S and DeiT-B
\cite{touvron2021training} calibrated with 128 ImageNet samples
(\cref{tab:vit-result}). CoreQ outperforms both GPTQ and GPTAQ at
comparable wall-clock cost: on DeiT-S it improves top-1 accuracy by $3.47$ over GPTQ and $1.32$ over GPTAQ, and on DeiT-B by $0.95$ and $0.37$ respectively. The bounded-search variants (4-CoreQ and 8-CoreQ)
provide further gains on DeiT-S at modest extra cost; on DeiT-B the marginal benefit
saturates at $K = 8$. These results indicate that CoreQ is not specific to causal language modeling, and it captures structure of the layerwise mismatch that arises in transformer architectures
generally.

\input{files/tab-vit}

\input{files/tab5}
\input{files/tab-70b}
\input{files/tab-instruct-model}

\clearpage
\input{appendix/extend-related-work}

\section{Broader Impacts}
\label{app:broader-impacts}

This work aims to improve the efficiency of large language model deployment through post-training quantization. By reducing memory footprint and inference cost, CoreQ may make LLMs more accessible on resource-constrained hardware and reduce the computational and energy cost of serving existing models. These benefits are especially relevant for deployment scenarios where full-precision inference is impractical. However, more efficient inference can also lower the barrier to misuse. Quantized models may be deployed more easily for spam generation, misinformation, automated manipulation, or other harmful applications. CoreQ does not introduce new model capabilities or alter the training data, but it can make existing capabilities cheaper to run. As a result, quantized models should be deployed with the same safety evaluations, access controls, and misuse-mitigation measures as their full-precision counterparts.

%% file: appendix/coreq_appendix_theory.tex
\section{Theoretical Results}
\label{app:theory}

Here we provide theoretical results and explanations introduced in Section~\ref{sec:methodology}. We first provide a population view showing
why an intermediate correction coefficient is natural.

\paragraph{Notation.}
Bold uppercase symbols denote deterministic matrices, including weights and finite calibration quantities. In particular,
\(\mathbf X_{\mathrm q},\mathbf X_{\mathrm f}\in\mathbb R^{n\times N}\)
denote empirical activation matrices formed from \(N\) calibration activation vectors. Non-bold uppercase symbols
\(X_{\mathrm q},X_{\mathrm f}\in\mathbb R^n\)
denote the corresponding population random activation vectors. Thus, expectations such as
\(\mathbb E[X_{\mathrm q}X_{\mathrm q}^{\top}]\)
are taken over population activations, whereas Frobenius norms involving
\(\mathbf X_{\mathrm q}\) or \(\mathbf X_{\mathrm f}\) are empirical calibration quantities.

\subsection{Oracle shrinkage for mismatch correction}
\label{app:oracle-shrinkage}

Let \(X_{\mathrm f},X_{\mathrm q}\in\mathbb R^n\) denote the full-precision-prefix and
quantized-prefix activation random vectors at deployment, and define
\[
    \bar{\mathbf H}:=\mathbb E[X_{\mathrm q}X_{\mathrm q}^{\top}],
    \qquad
    \bar{\mathbf C}:=\mathbb E[X_{\mathrm f}X_{\mathrm q}^{\top}].
\]
Assume \(\bar{\mathbf H}\succ 0\). Writing \(X_{\mathrm f}=X_{\mathrm q}+\Delta X\), define
\[
    \bar{\boldsymbol\Delta}:=\mathbb E[\Delta X\,X_{\mathrm q}^{\top}],
    \qquad
    \bar{\mathbf S}:=\mathbf W\bar{\boldsymbol\Delta}\bar{\mathbf H}^{-1}.
\]
Let \(\widetilde{\mathbf S}\) be a finite-sample calibration estimate of \(\bar{\mathbf S}\),
and consider the scalar family of continuous centers
\[
    \mathbf W^{(\alpha)}:=\mathbf W+\alpha\widetilde{\mathbf S},
    \qquad
    \alpha\in[0,1].
\]

\begin{theorem}[Oracle shrinkage for mismatch correction]
\label{thm:population-alpha}
The population layer-wise output error
\[
    \mathcal L_{\mathrm{pop}}(\widehat{\mathbf W})
    :=
    \mathbb E\|\mathbf W X_{\mathrm f}-\widehat{\mathbf W}X_{\mathrm q}\|_2^2
\]
is minimized over continuous weights by
\[
    \mathbf W^\star = \mathbf W+\bar{\mathbf S}.
\]
Assume that the finite-sample correction estimate is unbiased,
\(\mathbb E_{\widetilde{\mathbf S}}[\widetilde{\mathbf S}]=\bar{\mathbf S}\). Ignoring discrete rounding,
the interpolated center \(\mathbf W^{(\alpha)}=\mathbf W+\alpha\widetilde{\mathbf S}\) has the expected excess population error
\[
    R(\alpha)
    :=
    \mathbb E_{\widetilde{\mathbf S}}
    \!\left[
        \mathcal L_{\mathrm{pop}}(\mathbf W^{(\alpha)})
        -
        \mathcal L_{\mathrm{pop}}(\mathbf W^\star)
    \right]
    =
    (1-\alpha)^2\mu^2+\alpha^2\sigma^2,
\]
where
\[
    \mu^2:=\|\bar{\mathbf S}\bar{\mathbf H}^{1/2}\|_F^2,
    \qquad
    \sigma^2:=
    \mathbb E_{\widetilde{\mathbf S}}
    \|(\widetilde{\mathbf S}-\bar{\mathbf S})\bar{\mathbf H}^{1/2}\|_F^2 .
\]
If \(\mu^2+\sigma^2>0\), then
\[
    \alpha^\star
    :=
    \arg\min_{\alpha\in[0,1]}R(\alpha)
    =
    \frac{\mu^2}{\mu^2+\sigma^2}.
\]
In particular, \(0<\alpha^\star<1\) whenever
\(\mu^2>0\) and \(\sigma^2>0\).
\end{theorem}

\begin{proof}
Expanding the population loss gives
\[
\begin{aligned}
    \mathcal L_{\mathrm{pop}}(\widehat{\mathbf W})
    &=
    \operatorname{tr}\!\left(\mathbf W\mathbb E[X_{\mathrm f}X_{\mathrm f}^{\top}]\mathbf W^{\top}\right)
    -2\operatorname{tr}\!\left(\mathbf W\bar{\mathbf C}\widehat{\mathbf W}^{\top}\right)
    +\operatorname{tr}\!\left(\widehat{\mathbf W}\bar{\mathbf H}\widehat{\mathbf W}^{\top}\right).
\end{aligned}
\]
Completing the square gives
\[
    \mathcal L_{\mathrm{pop}}(\widehat{\mathbf W})
    -
    \mathcal L_{\mathrm{pop}}(\mathbf W^\star)
    =
    \|(\widehat{\mathbf W}-\mathbf W^\star)\bar{\mathbf H}^{1/2}\|_F^2 .
\]
Since \(X_{\mathrm f}=X_{\mathrm q}+\Delta X\),
\[
    \bar{\mathbf C}
    =
    \bar{\mathbf H}+\bar{\boldsymbol\Delta},
\]
so
\[
    \mathbf W\bar{\mathbf C}\bar{\mathbf H}^{-1}
    =
    \mathbf W(\bar{\mathbf H}+\bar{\boldsymbol\Delta})\bar{\mathbf H}^{-1}
    =
    \mathbf W+\mathbf W\bar{\boldsymbol\Delta}\bar{\mathbf H}^{-1}
    =
    \mathbf W+\bar{\mathbf S}.
\]
Thus \(\mathbf W^\star=\mathbf W+\bar{\mathbf S}\).

For the interpolated center \(\mathbf W^{(\alpha)}=\mathbf W+\alpha\widetilde{\mathbf S}\),
\[
    \mathbf W^{(\alpha)}-\mathbf W^\star
    =
    \mathbf W+\alpha\widetilde{\mathbf S}-(\mathbf W+\bar{\mathbf S})
    =
    -(1-\alpha)\bar{\mathbf S}+\alpha(\widetilde{\mathbf S}-\bar{\mathbf S}).
\]
Therefore,
\[
\begin{aligned}
    R(\alpha)
    &=
    \mathbb E_{\widetilde{\mathbf S}}
    \left[
        \| (\mathbf W^{(\alpha)}-\mathbf W^\star)\bar{\mathbf H}^{1/2}\|_F^2
    \right] \\
    &=
    (1-\alpha)^2\|\bar{\mathbf S}\bar{\mathbf H}^{1/2}\|_F^2
    +
    \alpha^2
    \mathbb E_{\widetilde{\mathbf S}}
    \|(\widetilde{\mathbf S}-\bar{\mathbf S})\bar{\mathbf H}^{1/2}\|_F^2 ,
\end{aligned}
\]
where the cross term vanishes because
\(\mathbb E_{\widetilde{\mathbf S}}[\widetilde{\mathbf S}-\bar{\mathbf S}]=0\).
Minimizing the resulting scalar quadratic gives
\[
    \alpha^\star
    =
    \frac{\mu^2}{\mu^2+\sigma^2}.
\]
\end{proof}

\cref{thm:population-alpha} explains why interpolation is needed:
\(\alpha=0\) ignores the population correction \(\bar{\mathbf S}\), while
\(\alpha=1\) fully trusts the finite-sample estimate \(\widetilde{\mathbf S}\).
The optimal scalar coefficient balances population correction energy
\(\mu^2\) against finite-calibration estimation error \(\sigma^2\), and is therefore
intermediate whenever both are present.

\subsection{Connection to the empirical criterion}
\label{app:geometric-proxy}

Let $\mathcal R:=\{\mathbf A\mathbf X_{\mathrm q}:
\mathbf A\in\mathbb R^{m\times n}\}$ and
$\mathbf D:=\mathbf W(\mathbf X_{\mathrm f}-\mathbf X_{\mathrm q})$.
The following decomposition uses the orthogonal projection of
$\mathbf D$ onto $\mathcal R$, so that
$\mathbf s=\widetilde{\mathbf S}\mathbf X_{\mathrm q}\in\mathcal R$
and $\boldsymbol\eta\perp\mathcal R$.

\begin{lemma}[Reliability-regularized empirical target risk]
\label{lem:d-proxy}
Let
\[
    \mathbf W^{(\alpha)}:=\mathbf W+\alpha\widetilde{\mathbf S},
    \qquad
    \mathbf Y_{\mathcal R}:=(\mathbf W+\widetilde{\mathbf S})\mathbf X_{\mathrm q},
    \qquad
    \mathbf{D}=\mathbf s+\boldsymbol{\eta},
    \qquad
    \mathbf s:=\widetilde{\mathbf S}\mathbf X_{\mathrm q},
    \quad
    \boldsymbol{\eta}\perp\mathcal R .
\]
Let
\[
    \mathbf X^{(\alpha)}
    :=
    (1-\alpha)\mathbf X_{\mathrm q}
    +
    \alpha\mathbf X_{\mathrm f}.
\]
The empirical plug-in analogue of the excess population error
in \cref{thm:population-alpha} is
\[
    \widehat R_{\mathrm{plug}}(\alpha)
    :=
    \|(\mathbf W^{(\alpha)}-(\mathbf W+\widetilde{\mathbf S}))\mathbf X_{\mathrm q}\|_F^2
    =
    (1-\alpha)^2\|\mathbf s\|_F^2 .
\]
Moreover, the displacement criterion used in CoreQ satisfies
\[
    \|\mathbf Y_{\mathcal R}-\mathbf W\mathbf X^{(\alpha)}\|_F^2
    =
    \widehat R_{\mathrm{plug}}(\alpha)
    +
    \|\mathbf W\mathbf X^{(\alpha)}-\mathbf W^{(\alpha)}\mathbf X_{\mathrm q}\|_F^2
    =
    (1-\alpha)^2\|\mathbf s\|_F^2+\alpha^2\|\boldsymbol{\eta}\|_F^2 .
\]
\end{lemma}

\begin{proof}
Since \(\mathbf W^{(\alpha)}\mathbf X_{\mathrm q}=\mathbf W\mathbf X_{\mathrm q}+\alpha\mathbf s\) and
\(\mathbf Y_{\mathcal R}=\mathbf W\mathbf X_{\mathrm q}+\mathbf s\),
\[
    \widehat R_{\mathrm{plug}}(\alpha)
    =
    \|\mathbf Y_{\mathcal R}-\mathbf W^{(\alpha)}\mathbf X_{\mathrm q}\|_F^2
    =
    (1-\alpha)^2\|\mathbf s\|_F^2 .
\]
Also,
\[
    \mathbf W\mathbf X^{(\alpha)}
    =
    \mathbf W\mathbf X_{\mathrm q}+\alpha\mathbf{D}
    =
    \mathbf W\mathbf X_{\mathrm q}+\alpha\mathbf s+\alpha\boldsymbol{\eta}
    =
    \mathbf W^{(\alpha)}\mathbf X_{\mathrm q}+\alpha\boldsymbol{\eta}.
\]
Therefore
\[
    \mathbf Y_{\mathcal R}-\mathbf W\mathbf X^{(\alpha)}
    =
    (\mathbf Y_{\mathcal R}-\mathbf W^{(\alpha)}\mathbf X_{\mathrm q})
    +
    (\mathbf W^{(\alpha)}\mathbf X_{\mathrm q}-\mathbf W\mathbf X^{(\alpha)})
    =
    (1-\alpha)\mathbf s-\alpha\boldsymbol{\eta}.
\]
The two terms are orthogonal because \(\mathbf s\in\mathcal R\) and
\(\boldsymbol{\eta}\perp\mathcal R\), giving the stated decomposition.
\end{proof}

\cref{thm:population-alpha} shows that, when the finite-calibration correction estimate is imperfect, the desired correction need not be full and the population-optimal coefficient balances the benefit of applying the correction against the cost of trusting an inaccurate estimate. This motivates shrinkage, but it does not by itself provide a computable coefficient, since the population correction energy and estimation error are unavailable in layer-wise PTQ.

The calibration-set analogue in Lemma~\ref{lem:d-proxy} measures only the distance to the finite-sample reachable target and therefore always prefers \(\alpha=1\). This is overly optimistic because it provides no penalty for the portion of the observed mismatch that cannot be produced by changing the current layer's weights. CoreQ therefore uses the orthogonal residual energy \(\|\boldsymbol{\eta}\|_F^2\) as an observable geometric penalty for such unreachable mismatch. This yields the reliability-regularized distance metric
\[
    \|\mathbf d(\alpha)\|_F^2
    =
    \|\mathbf Y_{\mathcal R}-\mathbf W\mathbf X^{(\alpha)}\|_F^2
    =
    (1-\alpha)^2\|\mathbf s\|_F^2+\alpha^2\|\boldsymbol{\eta}\|_F^2,
\]
whose minimizer is
\[
    \alpha_{\mathrm{corr}}
    =
    \frac{\|\mathbf s\|_F^2}
         {\|\mathbf s\|_F^2+\|\boldsymbol{\eta}\|_F^2}.
\]
Thus, \(\alpha_{\mathrm{corr}}\) measures the empirical fraction of observed mismatch energy that lies in the reachable response space of the current layer. It shrinks the correction when a large portion of the mismatch is structurally unexplained by the current layer, which helps avoid over-trusting calibration-specific mismatch patterns.

\subsection{Extension to biased correction estimators}
\label{app:biased-correction}

The unbiasedness assumption in \cref{thm:population-alpha} need not hold
for an estimator involving a random matrix inverse. We first derive the
expected excess population error and optimal correction coefficient when
the calibration estimator may be biased. We then bound the deviation from
the coefficient given by the unbiased formula and the additional risk
incurred by retaining that formula. Finally, we use an exact identity for
the correction estimator to relate these bounds to calibration size,
establishing an $O(N^{-1})$ coefficient gap and an $O(N^{-2})$ risk penalty
under the stated sampling, moment, and stability assumptions. Throughout
the analysis, the layer weights and prefixes remain fixed across repeated
calibration draws.

\subsubsection{Population risk with estimator bias}
\label{app:biased-risk}

Let $\mathbf X_{\mathrm f},\mathbf X_{\mathrm q}\in\mathbb{R}^{n\times N}$ denote the full-precision and
quantized-prefix activation matrices for a randomly drawn calibration set of
size $N$. Define the corresponding output-mismatch matrix
\[
\mathbf D:=\mathbf W(\mathbf X_{\mathrm f}-\mathbf X_{\mathrm q})\in\mathbb{R}^{m\times N}.
\]
The normalized empirical moments and finite-calibration correction are
\[
\widehat{\mathbf H}_N:=\frac{1}{N}\mathbf X_{\mathrm q}\mathbf X_{\mathrm q}^\top,
\qquad
\widehat{\mathbf M}_N:=\frac{1}{N}\mathbf D\mathbf X_{\mathrm q}^\top,
\qquad
\widetilde{\mathbf S}_N:=\widehat{\mathbf M}_N\widehat{\mathbf H}_N^{-1}.
\]
The population moments and population correction are
\[
\bar{\mathbf H}:=\mathbb{E}[\widehat{\mathbf H}_N]\succ0,
\qquad
\bar{\mathbf M}:=\mathbb{E}[\widehat{\mathbf M}_N],
\qquad
\bar{\mathbf S}:=\bar{\mathbf M}\bar{\mathbf H}^{-1}.
\]
All expectations are over repeated random calibration-set draws. With the population moments defined above,
$\bar{\mathbf M}=\mathbf W(\bar{\mathbf C}-\bar{\mathbf H})$, and equivalently
$\bar{\mathbf S}=\mathbf W(\bar{\mathbf C}\bar{\mathbf H}^{-1}-\mathbf I)$.

For matrices $\mathbf A$ and $\mathbf B$, define the $\bar{\mathbf H}$-weighted inner product
\[
\langle \mathbf A,\mathbf B\rangle_{\bar{\mathbf H}}
:=\operatorname{tr}(\mathbf A\bar{\mathbf H} \mathbf B^\top),
\qquad
\|\mathbf A\|_{\bar{\mathbf H}}^2:=\langle \mathbf A,\mathbf A\rangle_{\bar{\mathbf H}}
=\|\mathbf A\bar{\mathbf H}^{1/2}\|_F^2.
\]

Define the finite-sample estimation error and its mean bias by
\[
\mathbf E_N:=\widetilde{\mathbf S}_N-\bar{\mathbf S},
\qquad
\mathbf B_N:=\mathbb{E}[\mathbf E_N]
=\mathbb{E}[\widetilde{\mathbf S}_N]-\bar{\mathbf S}.
\]
Also define
\[
\mu^2:=\|\bar{\mathbf S}\bar{\mathbf H}^{1/2}\|_F^2,
\qquad
\sigma_N^2:=\mathbb{E}\|\mathbf E_N\bar{\mathbf H}^{1/2}\|_F^2,
\]
and the correction-aligned bias
\[
\rho_N
:=\langle \bar{\mathbf S},\mathbf B_N\rangle_{\bar{\mathbf H}}
=\operatorname{tr}(\bar{\mathbf S}\bar{\mathbf H} \mathbf B_N^\top).
\]
Here $\sigma_N^2$ is the mean-square estimation error around $\bar{\mathbf S}$, which includes both variance and squared bias.

\begin{theorem}[Oracle coefficient with estimator bias]
\label{thm:biased-oracle}
By the completed-square identity in the proof of
\cref{thm:population-alpha}, the expected excess population error for
$\mathbf W+\alpha\widetilde{\mathbf S}_N$ is
\[
R_N(\alpha)
:=\mathbb{E}\left\|
(\alpha\widetilde{\mathbf S}_N-\bar{\mathbf S})\bar{\mathbf H}^{1/2}
\right\|_F^2.
\]
Then
\[
R_N(\alpha)
=(1-\alpha)^2\mu^2
+\alpha^2\sigma_N^2
-2\alpha(1-\alpha)\rho_N.
\]
If $\mu^2+\sigma_N^2+2\rho_N>0$, the unconstrained minimizer is
\[
\alpha_{N,\mathrm{bias}}^\star
=
\frac{\mu^2+\rho_N}
{\mu^2+\sigma_N^2+2\rho_N}.
\]
The constrained minimizer over $[0,1]$ is obtained by projecting this value
onto $[0,1]$.
\end{theorem}

\begin{proof}
Since
\[
\alpha\widetilde{\mathbf S}_N-\bar{\mathbf S}
=-(1-\alpha)\bar{\mathbf S}+\alpha\mathbf E_N,
\]
expanding the squared $\bar{\mathbf H}$-norm gives
\begin{align*}
R_N(\alpha)
&=(1-\alpha)^2\|\bar{\mathbf S}\|_{\bar{\mathbf H}}^2
+\alpha^2\mathbb{E}\|\mathbf E_N\|_{\bar{\mathbf H}}^2
-2\alpha(1-\alpha)
\mathbb{E}\langle \bar{\mathbf S},\mathbf E_N\rangle_{\bar{\mathbf H}}.
\end{align*}
The three terms are $\mu^2$, $\sigma_N^2$, and $\rho_N$, respectively.
Differentiating the resulting quadratic yields the stated minimizer.
\end{proof}

\cref{thm:population-alpha} assumes $\mathbb{E}[\widetilde{\mathbf S}_N]=\bar{\mathbf S}$, which
implies $\mathbf B_N=0$ and hence $\rho_N=0$. The bias-aware theorem shows that the
original coefficient remains exact under the weaker condition $\rho_N=0$, even when $\mathbf B_N\neq0$. Only the aligned mean bias contributes to the cross term; all bias
components are also accounted for through $\sigma_N^2$.

Let
\[
M_N:=\mu^2+\sigma_N^2,
\qquad
\alpha_{N,0}^\star:=\frac{\mu^2}{M_N},
\]
where $M_N>0$ and $\alpha_{N,0}^\star$ is the coefficient from the
unbiased formula. When $M_N+2\rho_N>0$, a direct calculation gives
\[
\alpha_{N,\mathrm{bias}}^\star-\alpha_{N,0}^\star
=
\frac{\rho_N(\sigma_N^2-\mu^2)}
{M_N(M_N+2\rho_N)}.
\]

\begin{lemma}[Robustness to aligned bias]
\label{lem:aligned-bias-robustness}
Suppose
\[
|\rho_N|\le \frac{\epsilon_N}{2}M_N,
\qquad
0\le\epsilon_N<1.
\]
Then
\[
\left|
\alpha_{N,\mathrm{bias}}^\star-\alpha_{N,0}^\star
\right|
\le
\frac{\epsilon_N}{2(1-\epsilon_N)}.
\]
Moreover,
\[
R_N(\alpha_{N,0}^\star)
-R_N(\alpha_{N,\mathrm{bias}}^\star)
\le
\frac{\epsilon_N^2}{4(1-\epsilon_N)}M_N.
\]
The same bounds remain valid after projection of the biased optimum onto
$[0,1]$.
\end{lemma}

\begin{proof}
Because $|\sigma_N^2-\mu^2|\le M_N$ and
$M_N+2\rho_N\ge(1-\epsilon_N)M_N$,
\[
\left|
\alpha_{N,\mathrm{bias}}^\star-\alpha_{N,0}^\star
\right|
\le
\frac{|\rho_N|}{(1-\epsilon_N)M_N}
\le
\frac{\epsilon_N}{2(1-\epsilon_N)}.
\]
The risk is a quadratic with curvature $M_N+2\rho_N$, so
\[
R_N(\alpha)-R_N(\alpha_{N,\mathrm{bias}}^\star)
=(M_N+2\rho_N)
(\alpha-\alpha_{N,\mathrm{bias}}^\star)^2.
\]
Using the exact coefficient difference and the lower bound on the denominator
produces the stated risk bound. Projection cannot increase the distance from
$\alpha_{N,0}^\star\in[0,1]$ or decrease the constrained minimum below the
unconstrained minimum.
\end{proof}

A small aligned bias perturbs the optimal coefficient at first order in
$\epsilon_N$, whereas the excess-risk cost of retaining the original
coefficient is second order in $\epsilon_N$. Thus, exact matrix unbiasedness
is not required for the original theorem to remain practically accurate.

\subsubsection{Dependence on calibration size}
\label{app:calibration-size}

We next quantify how $\rho_N$ depends on the calibration size. The
argument separates two facts:
(i) empirical moments fluctuate at order $N^{-1/2}$, and
(ii) the ratio-estimator bias arises from products of these centered moment
errors, hence at second order.

Define the whitened empirical-moment errors
\[
\mathbf G_N
:=\bar{\mathbf H}^{-1/2}(\widehat{\mathbf H}_N-\bar{\mathbf H})\bar{\mathbf H}^{-1/2},
\qquad
\mathbf F_N
:=(\widehat{\mathbf M}_N-\bar{\mathbf M})\bar{\mathbf H}^{-1/2},
\]
and let
\[
\mathbf A:=\bar{\mathbf S}\bar{\mathbf H}^{1/2}.
\]
Then
\[
\widehat{\mathbf H}_N
=\bar{\mathbf H}^{1/2}(\mathbf I+\mathbf G_N)\bar{\mathbf H}^{1/2}.
\]

\begin{lemma}[Exact ratio-estimator identity]
\label{lem:ratio-estimator-identity}
Whenever $\widehat{\mathbf H}_N$ is invertible,
\[
(\widetilde{\mathbf S}_N-\bar{\mathbf S})\bar{\mathbf H}^{1/2}
=(\mathbf F_N-\mathbf A\mathbf G_N)(\mathbf I+\mathbf G_N)^{-1}.
\]
\end{lemma}

\begin{proof}
Since
\[
\widehat{\mathbf H}_N^{-1}
=\bar{\mathbf H}^{-1/2}(\mathbf I+\mathbf G_N)^{-1}\bar{\mathbf H}^{-1/2}
\]
and
\[
\widehat{\mathbf M}_N\bar{\mathbf H}^{-1/2}=\mathbf A+\mathbf F_N,
\]
we have
\[
\widetilde{\mathbf S}_N\bar{\mathbf H}^{1/2}
=(\mathbf A+\mathbf F_N)(\mathbf I+\mathbf G_N)^{-1}.
\]
Subtracting $\mathbf A=\bar{\mathbf S}\bar{\mathbf H}^{1/2}$ gives
\[
(\mathbf A+\mathbf F_N)(\mathbf I+\mathbf G_N)^{-1}-\mathbf A
=(\mathbf F_N-\mathbf A\mathbf G_N)(\mathbf I+\mathbf G_N)^{-1}.
\]
\end{proof}
To establish the $1/N$ rate, we impose the following assumptions.

\begin{assumption}[Independent sampling units]
\label{ass:independent-calibration}
The $N$ columns of $(\mathbf X_{\mathrm f},\mathbf X_{\mathrm q})$ are independent and identically distributed
paired sampling units. Equivalently, if $\mathbf x_{\mathrm q,i}$ and $\mathbf d_i$ denote the
$i$th columns of $\mathbf X_{\mathrm q}$ and $\mathbf D$, then
$\{(\mathbf d_i,\mathbf x_{\mathrm q,i})\}_{i=1}^N$ are i.i.d.
\end{assumption}

\begin{assumption}[Finite moment variances]
\label{ass:finite-moment-variances}
The centered single-sample moment contributions
\[
\mathbf G_i
:=\bar{\mathbf H}^{-1/2}
(\mathbf x_{\mathrm q,i}\mathbf x_{\mathrm q,i}^\top-\bar{\mathbf H})
\bar{\mathbf H}^{-1/2},
\]
\[
\mathbf F_i
:=(\mathbf d_i\mathbf x_{\mathrm q,i}^\top-\bar{\mathbf M})\bar{\mathbf H}^{-1/2}
\]
have finite second moments. Define
\[
v_H:=\mathbb{E}\|\mathbf G_i\|_F^2<\infty,
\qquad
v_M:=\mathbb{E}\|\mathbf F_i\|_F^2<\infty.
\]

\end{assumption}

\begin{assumption}[Uniform covariance stability]
\label{ass:covariance-stability}
For some $\delta\in[0,1)$,
\[
\|\mathbf G_N\|_2\le\delta
\]
almost surely over calibration draws. Equivalently, the empirical
covariance satisfies
\[
(1-\delta)\bar{\mathbf H}
\preceq
\widehat{\mathbf H}_N
\preceq
(1+\delta)\bar{\mathbf H}.
\]
\end{assumption}

Assumptions~\ref{ass:independent-calibration} and~\ref{ass:finite-moment-variances} justify the $1/N$ moment scaling rather than merely
postulating it. Indeed,
\[
\mathbf F_N=\frac{1}{N}\sum_{i=1}^N \mathbf F_i,
\qquad
\mathbf G_N=\frac{1}{N}\sum_{i=1}^N \mathbf G_i,
\]
with $\mathbb{E}[\mathbf F_i]=\mathbb{E}[\mathbf G_i]=0$. Independence eliminates all
cross terms, so
\[
\mathbb{E}\|\mathbf F_N\|_F^2=\frac{v_M}{N},
\qquad
\mathbb{E}\|\mathbf G_N\|_2^2
\le
\mathbb{E}\|\mathbf G_N\|_F^2
=\frac{v_H}{N}.
\]
The covariance-stability assumption is separate: it controls the random
inverse. Under bounded or sub-Gaussian sampling units, standard covariance
concentration makes this event high probability once the calibration size is
large relative to the effective activation dimension.

\begin{theorem}[Finite calibration bias bound]
\label{thm:finite-calibration-bias}
Under assumptions~\ref{ass:independent-calibration}--\ref{ass:covariance-stability},
\[
\|\mathbf B_N\bar{\mathbf H}^{1/2}\|_F
\le
\frac{
\sqrt{v_Mv_H}+\mu v_H
}{(1-\delta)N},
\qquad
\mu:=\|\bar{\mathbf S}\bar{\mathbf H}^{1/2}\|_F.
\]
Consequently,
\[
|\rho_N|
\le
\frac{
\mu\bigl(\sqrt{v_Mv_H}+\mu v_H\bigr)
}{(1-\delta)N}.
\]
\end{theorem}

\begin{proof}
Let $\mathbf Q_N:=\mathbf F_N-\mathbf A\mathbf G_N$. Since $\mathbb{E}[\mathbf F_N]=\mathbb{E}[\mathbf G_N]=0$,
$\mathbb{E}[\mathbf Q_N]=0$. Using
\[
\mathbf Q_N(\mathbf I+\mathbf G_N)^{-1}
=\mathbf Q_N-\mathbf Q_N\mathbf G_N(\mathbf I+\mathbf G_N)^{-1},
\]
the exact identity implies
\[
\mathbf B_N\bar{\mathbf H}^{1/2}
=-\mathbb{E}\left[
\mathbf Q_N\mathbf G_N(\mathbf I+\mathbf G_N)^{-1}
\right].
\]
Assumption~\ref{ass:covariance-stability} gives
$\|(\mathbf I+\mathbf G_N)^{-1}\|_2\le(1-\delta)^{-1}$. Therefore,
\begin{align*}
\|\mathbf B_N\bar{\mathbf H}^{1/2}\|_F
&\le
\frac{1}{1-\delta}
\mathbb{E}\left[
(\|\mathbf F_N\|_F+\mu\|\mathbf G_N\|_2)\|\mathbf G_N\|_2
\right]\\
&\le
\frac{1}{1-\delta}
\left(
\sqrt{\mathbb{E}\|\mathbf F_N\|_F^2
\,\mathbb{E}\|\mathbf G_N\|_2^2}
+
\mu\mathbb{E}\|\mathbf G_N\|_2^2
\right)\\
&\le
\frac{\sqrt{v_Mv_H}+\mu v_H}{(1-\delta)N}.
\end{align*}
Finally, weighted Cauchy--Schwarz yields
\[
|\rho_N|
=|\langle \bar{\mathbf S},\mathbf B_N\rangle_{\bar{\mathbf H}}|
\le
\|\bar{\mathbf S}\bar{\mathbf H}^{1/2}\|_F
\|\mathbf B_N\bar{\mathbf H}^{1/2}\|_F,
\]
which gives the second bound.
\end{proof}

The constants $v_H$ and $v_M$ contain the dimensional and distributional
dependence. If $\mathbf u_i:=\bar{\mathbf H}^{-1/2}\mathbf x_{\mathrm q,i}\in\mathbb{R}^n$ is the whitened
activation, then
\[
v_H
=\mathbb{E}\|\mathbf u_i\mathbf u_i^\top-\mathbf I\|_F^2
=\mathbb{E}\|\mathbf u_i\|_2^4-n.
\]
Thus, $v_H$ reflects the activation dimension and its fourth moments.
Similarly,
\[
v_M
=\mathbb{E}\left\|
(\mathbf d_i\mathbf x_{\mathrm q,i}^\top-\bar{\mathbf M})\bar{\mathbf H}^{-1/2}
\right\|_F^2
\]
depends on the output dimension, activation dimension, and mixed fourth
moments of the mismatch and activations. The constants therefore retain the dimension and distribution dependence
of the moment fluctuations.

Define
\[
\epsilon_N:=\frac{2|\rho_N|}{M_N},
\qquad
M_N:=\mu^2+\sigma_N^2.
\]
If $M_N$ is bounded away from zero and the assumptions above hold with
constants independent of $N$, then
\[
\epsilon_N=O(N^{-1}).
\]
\cref{lem:aligned-bias-robustness} therefore gives
\[
\left|
\alpha_{N,\mathrm{bias}}^\star-\alpha_{N,0}^\star
\right|=O(N^{-1}),
\]
while
\[
R_N(\alpha_{N,0}^\star)
-R_N(\alpha_{N,\mathrm{bias}}^\star)
=O(N^{-2}).
\]
Hence the ratio-estimator bias may perturb the optimal coefficient at order
$1/N$, but the resulting excess-risk penalty from retaining the original
coefficient is second order, at order $1/N^2$.

%% file: files/tab-matrix-dim.tex
\begin{table}[h]
\centering
\small
\caption{Dimensions of major tensors used in CoreQ and $K$-CoreQ rounding for
one linear layer of shape $m \times n$. $B$ is the block size, $K$ is the beam
width, and $A$ is the number of quantization levels. The
center correction $\mathbf{T}$ used in lazy-batch update is defined in
Algorithms~\ref{alg:coreq-lazy-batch} and~\ref{alg:kcoreq-lazy-batch}.}
\vspace{0.05 in}
\label{tab:beam_dims}
\resizebox{0.7\linewidth}{!}{
\begin{tabular}{l|c|c}
\bottomrule
Tensor & CoreQ ($K{=}1$) & $K$-CoreQ (beam) \\
\hline
Corrected rounding center $\mathbf{W}_{\mathrm{corr}}$ & $m \times n$ & $m \times n$ \\
Cholesky factor $\mathbf{L}$ & $n \times n$ & $n \times n$ \\
Quantized weight $\mathbf{Q}$ & $m \times n$ & $m \times n$ \\
\hline
Beam scores $\mathbf{S}$ & --- & $m \times K$ \\
Suffix decisions $\mathbf{Q}_{\mathrm{tail}}$ & --- & up to $m \times K \times n$ \\
Per-step expansion cost $\boldsymbol{\Delta}$ & --- & $m \times K \times A$ \\
Beam ancestry $\bm{\pi}$ & --- & $m \times K$ \\
\hline
Block decisions $\mathbf{Q}_{\mathrm{blk}}$ & $m \times B$ & $m \times K \times B$ \\
Center correction $\mathbf{T}$ & $m \times B$ & $m \times K \times B$ \\
Interference-cancelled center $\widetilde{\mathbf{W}}$ & $m \times B$ & $m \times K$ \\
\toprule
\end{tabular}
}
\end{table}

%% file: files/alg_alpha.tex
\begin{figure}[h]
\makebox[\columnwidth][c]{%
\begin{minipage}[h]{0.9\linewidth}
\begin{algorithm}[H]
\DontPrintSemicolon
\caption{$\alpha_{\mathrm{corr}}$: closed-form correction coefficient}
\label{alg:alpha-corr}
\SetKwInput{KwData}{Input}
\SetKwInput{KwResult}{Output}
\KwData{FP weight $\mathbf{W}\in\mathbb{R}^{m\times n}$,
teacher input $\mathbf{X}_f\in\mathbb{R}^{n\times N}$,
student input $\mathbf{X}_q\in\mathbb{R}^{n\times N}$,
damping $\lambda\ge 0$}
\KwResult{Correction coefficient $\alpha_{\mathrm{corr}}\in[0,1]$,
shifted target $\mathbf{W}_\mathrm{corr}\in\mathbb{R}^{m\times n}$}

$\mathbf{H} \leftarrow \mathbf{X}_q\mathbf{X}_q^\top + \lambda\mathbf{I}$
\tcp*{Hessian with damping}
$\mathbf{U} \leftarrow \mathbf{W}(\mathbf{X}_f - \mathbf{X}_q)$
\tcp*{layerwise mismatch}
$\widetilde{\mathbf{S}} \leftarrow \mathbf{U}\mathbf{X}_q^\top\mathbf{H}^{-1}$
\tcp*{empirical reachable correction}
$\mathbf{s} \leftarrow \widetilde{\mathbf{S}}\mathbf{X}_q$
\tcp*{reachable component}
$\mathbf{\eta} \leftarrow \mathbf{U} - \mathbf{s}$
\tcp*{residual}
$\alpha_{\mathrm{corr}} \leftarrow
    \dfrac{\|\mathbf{s}\|_F^2}{\|\mathbf{s}\|_F^2 + \|\mathbf{\eta}\|_F^2}$
    
$\mathbf{W}_\mathrm{corr} \leftarrow \mathbf{W} + \alpha_{\mathrm{corr}}\widetilde{\mathbf{S}}$
\tcp*{shifted rounding center}
\Return{$(\alpha_{\mathrm{corr}},\, \mathbf{W}_\mathrm{corr})$}
\end{algorithm}
\end{minipage}
}
\end{figure}

%% file: files/lazy-batch-algorithm.tex
\begin{figure}[h]
\makebox[\columnwidth][c]{%
\begin{minipage}[h]{0.9\linewidth}
\begin{algorithm}[H]
\DontPrintSemicolon
\caption{CoreQ: greedy successive rounding with lazy-batch update}
\label{alg:coreq-lazy-batch}
\SetKwInput{KwData}{Input}
\SetKwInput{KwResult}{Output}
\KwData{FP weight $\mathbf{W}\in\mathbb{R}^{m\times n}$,
teacher input $\mathbf{X}_f$, student input $\mathbf{X}_q$,
permutation $\mathbf{P}$, block size $B$, damping $\lambda\ge 0$}
\KwResult{Quantized weight $\mathbf{Q}$}

$(\alpha_{\mathrm{corr}},\, \mathbf{W}_\mathrm{corr}) \leftarrow
    \textsc{ComputeAlpha}(\mathbf{W}, \mathbf{X}_f, \mathbf{X}_q, \lambda)$
\tcp*{Algorithm~\ref{alg:alpha-corr}}

$\mathbf{H} \leftarrow \mathbf{P}^\top(\mathbf{X}_q\mathbf{X}_q^\top + \lambda\mathbf{I})\mathbf{P}$

$\mathbf{L} \leftarrow \textsc{Cholesky}(\mathbf{H})$

$\mathbf{L} \leftarrow \mathbf{L}/\mathrm{diag}(\mathbf{L}) - \mathbf{I}$

$\mathbf{W}_\mathrm{corr} \leftarrow \mathbf{W}_\mathrm{corr}\mathbf{P}$
\tcp*{permute target}
$\mathbf{Q} \in \mathbb{R}^{m\times n} \leftarrow \mathbf{0}$

\For{$i = n,\;n-B,\;n-2B,\dots$}{
    $\mathbf{T} \leftarrow
        (\mathbf{W}_\mathrm{corr}[:, i:] - \mathbf{Q}[:, i:])\,
        \mathbf{L}[i:,\, i\!-\!B:i] \in \mathbb{R}^{m\times B}$
    \tcp*{tail correction}
    $\widetilde{\mathbf{W}} \in \mathbb{R}^{m\times B} \leftarrow \mathbf{0}$
    
    \For{$j=B-1,B-2,\dots,0$}{
        $t \leftarrow i - B + j$
        \tcp*{absolute column index}
        $\widetilde{\mathbf{W}}[:, j] \leftarrow
            \mathbf{W}_\mathrm{corr}[:, t] +
            (\mathbf{W}_\mathrm{corr}[:, i\!-\!B:i] - \mathbf{Q}[:, i\!-\!B:i])\,
            \mathbf{L}[i\!-\!B:i,\, t] +
            \mathbf{T}[:, j]$
            
        $\mathbf{Q}[:, t] \leftarrow
            \textsc{Round}(\widetilde{\mathbf{W}}[:, j])$
    }
}
$\mathbf{Q} \leftarrow \mathbf{Q}\mathbf{P}^\top$
\tcp*{undo permutation}
\Return{$\mathbf{Q}$}
\end{algorithm}
\end{minipage}
}
\end{figure}

%% file: files/algorithm2.tex
\begin{figure}[h]
\makebox[\columnwidth][c]{%
\begin{minipage}[h]{0.98\linewidth}
\begin{algorithm}[H]
\DontPrintSemicolon
\caption{$K$-CoreQ: beam-search rounding with lazy-batch update}
\label{alg:kcoreq-lazy-batch}
\SetKwInput{KwData}{Input}
\SetKwInput{KwResult}{Output}
\KwData{FP weight $\mathbf{W}\in\mathbb{R}^{m\times n}$,
teacher input $\mathbf{X}_f$, student input $\mathbf{X}_q$,
permutation $\mathbf{P}$, block size $B$, beam width $K$,
damping $\lambda\ge 0$}
\KwResult{Quantized weight $\mathbf{Q}$}

$(\alpha_{\mathrm{corr}},\, \mathbf{W}_\mathrm{corr}) \leftarrow
    \textsc{ComputeAlpha}(\mathbf{W}, \mathbf{X}_f, \mathbf{X}_q, \lambda)$
\tcp*{Algorithm~\ref{alg:alpha-corr}}
$\mathbf{H} \leftarrow \mathbf{P}^\top(\mathbf{X}_q\mathbf{X}_q^\top + \lambda\mathbf{I})\mathbf{P}$

$\mathbf{L} \leftarrow \textsc{Cholesky}(\mathbf{H})$

$\mathbf{L} \leftarrow \mathbf{L}/\mathrm{diag}(\mathbf{L}) - \mathbf{I}$

$\mathbf{W}_\mathrm{corr} \leftarrow \mathbf{W}_\mathrm{corr}\mathbf{P}$
\tcp*{permute target}

\textbf{Initialize beam:}

$\mathbf{S} \in \mathbb{R}^{m\times K} \leftarrow +\infty,\;
 \mathbf{S}_{:,0} \leftarrow 0$
\tcp*{beam scores (per row)}

$\mathbf{Q}_{\rm tail} \in \mathbb{R}^{m\times K\times 0} \leftarrow \emptyset$
\tcp*{accumulated suffix decisions}

\For{$i = n,\;n-B,\;n-2B,\dots$}{
    $\mathbf{T} \leftarrow
        (\mathbf{W}_\mathrm{corr}[:, i:] - \mathbf{Q}_{\rm tail})\,
        \mathbf{L}[i:,\, i\!-\!B:i]
        \in \mathbb{R}^{m\times K\times B}$
    \tcp*{tail correction}
    
    $\mathbf{Q}_{\rm blk} \in \mathbb{R}^{m\times K\times B} \leftarrow \mathbf{0}$
    
    $\bm{\pi} \in \mathbb{R}^{m\times K} \leftarrow [0, 1, \dots, K-1]$
    \tcp*{beam ancestry}
    
    \For{$j=B-1,B-2,\dots,0$}{
        $t \leftarrow i - B + j$
        \tcp*{absolute column index}
        
        $\widetilde{\mathbf{W}} \leftarrow
            \mathbf{W}_\mathrm{corr}[:, t] +
            (\mathbf{W}_\mathrm{corr}[:, i\!-\!B:i] - \mathbf{Q}_{\rm blk})\,
            \mathbf{L}[i\!-\!B:i,\, t] +
            \mathbf{T}[:, :, j]
            \in \mathbb{R}^{m\times K}$
            
        $\mathbf{V} \leftarrow \textsc{Levels}(\widetilde{\mathbf{W}})
            \in \mathbb{R}^{m\times A}$
        \tcp*{candidate quantization levels}
        
        $\boldsymbol{\Delta} \leftarrow \mathbf{L}[t, t]^2 \,
            |\widetilde{\mathbf{W}} - \mathbf{V}|^2
            \in \mathbb{R}^{m\times K\times A}$
            
        $(\mathbf{S},\, \bm{parent},\, \bm{choice}) \leftarrow
            \textsc{TopK}(\mathbf{S} + \boldsymbol{\Delta},\; K)$
        \tcp*{keep best $K$}
        
        $\mathbf{Q}_{\rm blk} \leftarrow
            \textsc{Gather}(\mathbf{Q}_{\rm blk}, \bm{parent})$
            
        $\mathbf{T} \leftarrow \textsc{Gather}(\mathbf{T}, \bm{parent})$
        
        $\bm{\pi} \leftarrow \textsc{Gather}(\bm{\pi}, \bm{parent})$
        
        $\mathbf{Q}_{\rm blk}[:, :, j] \leftarrow
            \textsc{PickLevel}(\mathbf{V}, \bm{choice})$
        \tcp*{selected level for column $t$}
    }
    $\mathbf{Q}_{\rm tail} \leftarrow
        \textsc{Concat}\!\big(\mathbf{Q}_{\rm blk},\;
                              \textsc{Gather}(\mathbf{Q}_{\rm tail}, \bm{\pi})\big)$
}
$k^\star_r \leftarrow \arg\min_{k \in [K]} \mathbf{S}_{r, k},\;\; \forall r \in [m]$
\tcp*{best beam per row}
$\mathbf{Q}_p[r, :] \leftarrow \mathbf{Q}_{\rm tail}[r,\, k^\star_r,\, :],\;\;
\forall r \in [m]$

$\mathbf{Q} \leftarrow \mathbf{Q}_p\mathbf{P}^\top$
\tcp*{undo permutation}
\Return{$\mathbf{Q}$}
\end{algorithm}
\end{minipage}
}
\end{figure}

%% file: files/fig-lim-asym.tex
\begin{figure}[h]
    \centering
    \includegraphics[width=0.9\linewidth]{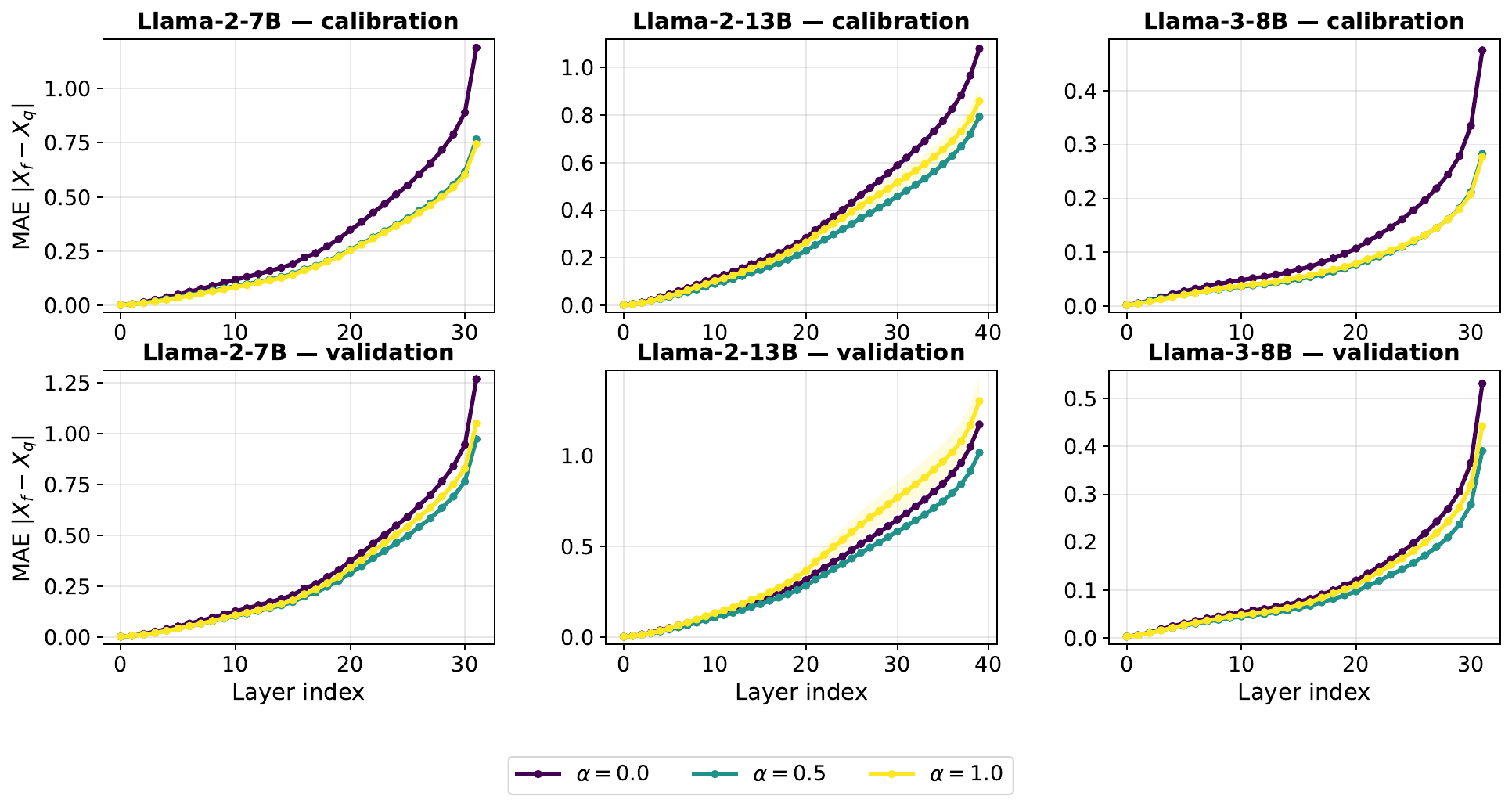}
    \caption{Mean activation error (MAE) measured on (top) calibration and (bottom) validation set. }
    \label{fig:alpha-cal-and-val}
\end{figure}

%% file: files/fig-alpha-sweep-llama.tex
\begin{figure}[t]
    \centering
    \begin{subfigure}[t]{1.0\linewidth}
        \centering
        \includegraphics[width=\linewidth]{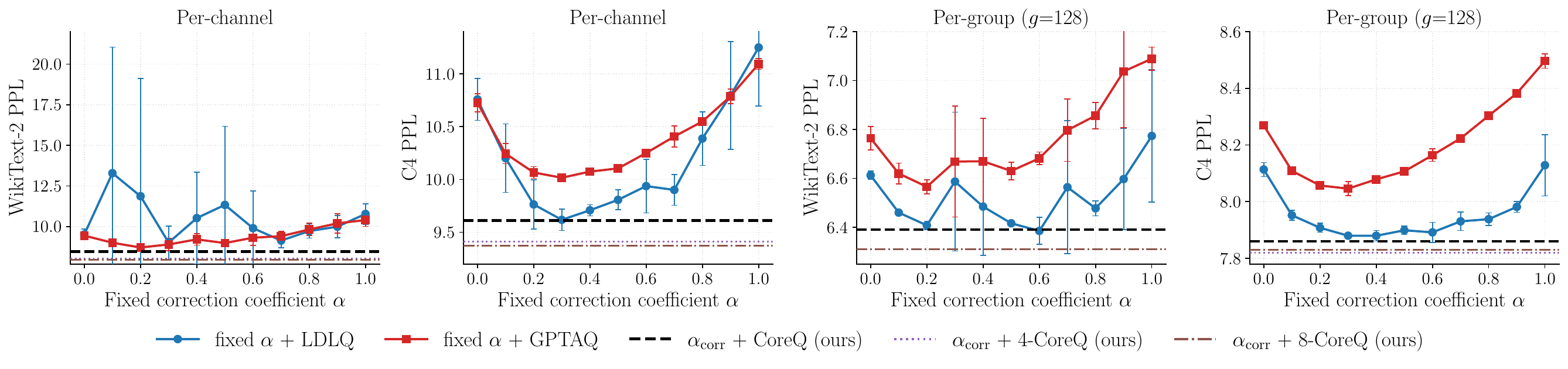}
        \vspace{-0.2in}
        \caption{LLaMA-2-7B}
        \label{fig:alpha-sweep-l2-7b}
    \end{subfigure}

    \vspace{0.05in}

    \begin{subfigure}[t]{1.0\linewidth}
        \centering
        \includegraphics[width=\linewidth]{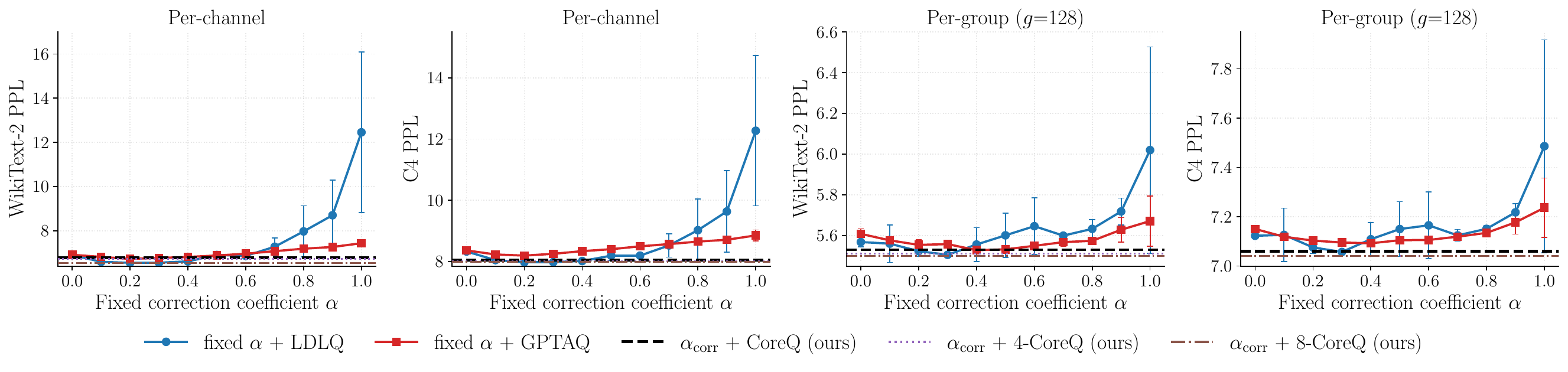}
        \vspace{-0.2in}
        \caption{LLaMA-2-13B}
        \label{fig:alpha-sweep-l2-13b}
    \end{subfigure}

    \vspace{0.05in}

    \begin{subfigure}[t]{1.0\linewidth}
        \centering
        \includegraphics[width=\linewidth]{figures/alpha_sweep_l3_8b.pdf}
        \vspace{-0.2in}
        \caption{LLaMA-3-8B}
        \label{fig:alpha-sweep-l3-8b}
    \end{subfigure}

    \caption{Fixed-$\alpha$ sweep on 3-bit weight-only quantization. Each panel reports perplexity on
    WikiText-2 or C4 under per-channel and per-group ($g{=}128$)
    granularity, mean$\,\pm\,$std over 5 calibration seeds. Solid lines
    sweep fixed $\alpha$ for LDLQ and GPTAQ; horizontal bands show
    CoreQ and its bounded-search variants $K$-CoreQ ($K{\in}\{4, 8\}$)
    using the closed-form $\alpha_{\mathrm{corr}}$.}
    \label{fig:alpha-sweep-extra}
\end{figure}

%% file: files/tab-cal-size.tex
\begin{table}[h]
\centering
\caption{Sensitivity to calibration set size. LLaMA-2-7B with 3-bit per-channel
weight-only quantization, calibrated on C4 (sequence length 2048). Perplexity
on Wikitext-2 (Wiki2) and C4 (lower is better), and average zero-shot accuracy
across six commonsense reasoning tasks (higher is better).}
\vspace{0.05 in}
\label{tab:calset-sensitivity}
\setlength{\tabcolsep}{4pt}
\resizebox{0.8\textwidth}{!}{%
\begin{tabular}{lcc|cc|cc}
\toprule
 & \multicolumn{2}{c}{Wiki2 PPL ($\downarrow$)}
 & \multicolumn{2}{c}{C4 PPL ($\downarrow$)}
 & \multicolumn{2}{c}{Avg. Acc. (\%, $\uparrow$)} \\
\cmidrule(lr){2-3} \cmidrule(lr){4-5} \cmidrule(lr){6-7}
\# Samples & GPTAQ & CoreQ & GPTAQ & CoreQ & GPTAQ & CoreQ \\
\midrule
16
  & 11.09 $\pm$ 0.34 & \cellcolor{blue!8} 11.03 $\pm$ 1.00
  & 12.13 $\pm$ 0.19 & \cellcolor{blue!8} 11.34 $\pm$ 0.14
  & 58.17 $\pm$ 0.48 & \cellcolor{blue!8} 59.28 $\pm$ 0.61 \\
32
  & 10.37 $\pm$ 1.35 & \cellcolor{blue!8}  9.17 $\pm$ 0.31
  & 10.80 $\pm$ 0.10 & \cellcolor{blue!8} 10.06 $\pm$ 0.07
  & 60.10 $\pm$ 0.23 & \cellcolor{blue!8} 60.96 $\pm$ 0.70 \\
64
  &  9.28 $\pm$ 0.57 & \cellcolor{blue!8}  9.06 $\pm$ 0.65
  & 10.26 $\pm$ 0.02 & \cellcolor{blue!8}  9.82 $\pm$ 0.22
  & 60.86 $\pm$ 0.64 & \cellcolor{blue!8} 61.01 $\pm$ 0.80 \\
128
  &  8.84 $\pm$ 0.23 & \cellcolor{blue!8}  8.41 $\pm$ 0.21
  & 10.04 $\pm$ 0.06 & \cellcolor{blue!8}  9.61 $\pm$ 0.12
  & 60.91 $\pm$ 0.52 & \cellcolor{blue!8} 61.46 $\pm$ 0.50 \\
\bottomrule
\end{tabular}%
}
\end{table}

%% file: files/tab-outlier-red.tex
\begin{table}[h]
\centering
\caption{W3 weight-only quantization: perplexity on Wiki2/C4 and wall-clock quantization time (seconds, single GPU). Perplexities are mean\,$\pm$\,std over calibration seeds; quantization time is mean over the same runs. }
\vspace{0.05 in}
\label{tab:outlier-woq}
\setlength{\tabcolsep}{3pt}
\resizebox{\textwidth}{!}{%
\begin{tabular}{l ccc ccc ccc}
\toprule
 & \multicolumn{3}{c}{LLaMA-2-7B} & \multicolumn{3}{c}{LLaMA-2-13B} & \multicolumn{3}{c}{LLaMA-3-8B} \\
\cmidrule(lr){2-4} \cmidrule(lr){5-7} \cmidrule(lr){8-10}
Method & Wiki2 & C4 & Time & Wiki2 & C4 & Time & Wiki2 & C4 & Time \\
\midrule
\multicolumn{10}{l}{\textit{Per-channel (W3)}} \\
\midrule
AWQ              & 16.93 $\pm$ 0.33 & 18.55 $\pm$ 0.45 & 801
                 &  6.48 $\pm$ 0.02 &  7.92 $\pm$ 0.02 & 1326
                 & 12.18 $\pm$ 0.16 & 16.48 $\pm$ 0.20 & 777 \\
OmniQuant        &  6.63 $\pm$ 0.01 &  8.64 $\pm$ 0.02 & 2309
                 &  5.60 $\pm$ 0.02 &  7.46 $\pm$ 0.03 & 3975
                 & 15.11 $\pm$ 0.47 & 21.10 $\pm$ 0.74 & 2338 \\
\rowcolor{blue!8} QuaRot + CoreQ
                 & \textbf{5.84 $\pm$ 0.01} & \textbf{7.70 $\pm$ 0.01} & \textbf{570}
                 & \textbf{5.17 $\pm$ 0.00} & \textbf{6.99 $\pm$ 0.00} & \textbf{956}
                 & \textbf{7.14 $\pm$ 0.01} & \textbf{11.09 $\pm$ 0.04} & \textbf{679} \\
\midrule
\multicolumn{10}{l}{\textit{Group-wise, $g{=}128$ (W3)}} \\
\midrule
AWQ              &  6.19 $\pm$ 0.01 &  \textbf{7.33 $\pm$ 0.01} & {843}
                 &  5.31 $\pm$ 0.01 &  \textbf{6.57 $\pm$ 0.01} & {1419}
                 &  8.18 $\pm$ 0.02 & 11.11 $\pm$ 0.01 & \textbf{838} \\
OmniQuant        &  6.10 $\pm$ 0.01 &  7.81 $\pm$ 0.01 & 2334
                 &  5.32 $\pm$ 0.00 &  7.02 $\pm$ 0.00 & 4147
                 &  8.81 $\pm$ 0.04 & 12.50 $\pm$ 0.02 & 2438 \\
\rowcolor{blue!8} QuaRot + CoreQ
                 & \textbf{5.78 $\pm$ 0.00} & {7.58 $\pm$ 0.00} & \textbf{787}
                 & \textbf{5.12 $\pm$ 0.00} & {6.91 $\pm$ 0.00} & \textbf{1282}
                 & \textbf{6.99 $\pm$ 0.01} & \textbf{10.74 $\pm$ 0.04} & 903 \\
\bottomrule
\end{tabular}%
}
\end{table}

%% file: files/tab-vit.tex
\begin{table}[h]
\centering
\caption{W4A4 quantization results of vision transformers, calibrated using 128 samples of ImageNet. All results are averaged over 5 random seeds.}
\vspace{0.05 in}
\label{tab:vit-result}
\resizebox{0.7\linewidth}{!}{
\begin{tabular}{l|cc|cc}
\bottomrule
\multirow{2}{*}{Method} & \multicolumn{2}{c|}{DeiT-S} & \multicolumn{2}{c}{DeiT-B} \\
 & Acc (\%) & Q.Time (s) & Acc (\%) & Q.Time (s) \\
\hline\hline
FP16   & 79.9              & --              & 81.8              & --              \\
\hline
GPTQ   &  71.94 $\pm$ 0.06  & 13.86 $\pm$ 2.09 & 77.60 $\pm$ 0.04  & 18.03 $\pm$ 0.30 \\
GPTAQ  &  74.09 $\pm$ 0.16  & 15.40 $\pm$ 0.70 & 78.18 $\pm$ 0.12  & 24.85 $\pm$ 0.75 \\
\rowcolor{blue!8}
CoreQ                  &  75.41 $\pm$ 0.07  & 14.75 $\pm$ 0.06 & 78.55 $\pm$ 0.09  & 21.20 $\pm$ 0.70 \\
\rowcolor{blue!8}
4-CoreQ &  \textbf{75.52} $\pm$ 0.08  & 17.41 $\pm$ 0.31 & 78.56 $\pm$ 0.12  & 24.74 $\pm$ 1.87 \\
\rowcolor{blue!8}
8-CoreQ &  \textbf{75.51} $\pm$ 0.07  & 17.62 $\pm$ 0.41 & \textbf{78.67} $\pm$ 0.10  & 25.18 $\pm$ 0.51 \\
\toprule
\end{tabular}}
\end{table}

%% file: files/tab5.tex
\begin{table}[h]
\centering
\caption{W3 weight-only quantization on Qwen3-8B and Phi-3-3.8B.
We report perplexity on WikiText-2 and C4 ($\downarrow$), average zero-shot
accuracy across six commonsense reasoning tasks ($\uparrow$), and wall-clock
quantization time (seconds). Mean$\,\pm\,$std over 5 calibration seeds. Best
per (granularity, model, metric) in bold.}
\vspace{0.05 in}
\label{tab:qwen-phi}
\setlength{\tabcolsep}{4pt}
\resizebox{\linewidth}{!}{%
\begin{tabular}{ll cccc | cccc}
\toprule
& & \multicolumn{4}{c|}{\textbf{Qwen3-8B}} & \multicolumn{4}{c}{\textbf{Phi-3-3.8B}} \\
\cmidrule(lr){3-6} \cmidrule(lr){7-10}
\textbf{Granularity} & \textbf{Method}
  & \textbf{Wiki2} ($\downarrow$)
  & \textbf{C4} ($\downarrow$)
  & \textbf{Avg.\,Acc} ($\uparrow$)
  & \textbf{Q.Time (s)}
  & \textbf{Wiki2} ($\downarrow$)
  & \textbf{C4} ($\downarrow$)
  & \textbf{Avg.\,Acc} ($\uparrow$)
  & \textbf{Q.Time (s)} \\
\midrule
& FP16 &  9.73 & 14.65 & 74.15 & ---
       &  6.01 &  9.11 & 75.41 & --- \\
\midrule
\multirow{4}{*}{Per-channel}
  & GPTQ   & 18.75 $\pm$ 0.45 & 19.69 $\pm$ 0.12 & 54.17 $\pm$ 0.51 & 499.2 $\pm$ 17.7
           & 15.46 $\pm$ 0.50 & 15.80 $\pm$ 0.05 & 54.48 $\pm$ 1.48 & 266.5 $\pm$  7.6 \\
  & GPTAQ  & 17.63 $\pm$ 0.59 & 19.08 $\pm$ 0.28 & 57.07 $\pm$ 0.90 & 647.2 $\pm$ 19.1
           & 12.99 $\pm$ 0.28 & 14.21 $\pm$ 0.15 & 57.19 $\pm$ 0.85 & 335.7 $\pm$  7.7 \\
  & LDLQ   & 16.18 $\pm$ 0.18 & 18.49 $\pm$ 0.08 & 56.19 $\pm$ 0.81 & 375.4 $\pm$ 13.4
           & 15.75 $\pm$ 0.74 & 16.03 $\pm$ 0.48 & 54.49 $\pm$ 0.65 & 216.1 $\pm$  2.8 \\
  & \cellcolor{blue!8} CoreQ
           & \cellcolor{blue!8} \textbf{15.44 $\pm$ 0.23}
           & \cellcolor{blue!8} \textbf{17.51 $\pm$ 0.08}
           & \cellcolor{blue!8} \textbf{63.58 $\pm$ 0.93}
           & \cellcolor{blue!8} 553.9 $\pm$ 12.0
           & \cellcolor{blue!8} \textbf{11.86 $\pm$ 0.34}
           & \cellcolor{blue!8} \textbf{13.68 $\pm$ 0.32}
           & \cellcolor{blue!8} \textbf{60.51 $\pm$ 1.30}
           & \cellcolor{blue!8} 302.1 $\pm$  5.8 \\
\midrule
\multirow{4}{*}{Per-group ($g128$)}
  & GPTQ   & 11.75 $\pm$ 0.25 & 15.09 $\pm$ 0.11 & 68.49 $\pm$ 0.51 & 494.9 $\pm$ 11.0
           &  8.84 $\pm$ 0.07 & 11.31 $\pm$ 0.07 & 67.84 $\pm$ 0.48 & 262.7 $\pm$  4.6 \\
  & GPTAQ  & 11.95 $\pm$ 0.08 & 15.13 $\pm$ 0.06 & 68.94 $\pm$ 0.42 & 670.4 $\pm$  9.7
           & \textbf{8.45 $\pm$ 0.02} & \textbf{10.97 $\pm$ 0.04} & \textbf{68.74 $\pm$ 0.29} & 332.0 $\pm$  4.0 \\
  & LDLQ   & \textbf{11.41 $\pm$ 0.05} & \textbf{14.65 $\pm$ 0.06} & 69.26 $\pm$ 0.78 & 419.5 $\pm$ 10.9
           &  9.07 $\pm$ 0.29 & 11.46 $\pm$ 0.25 & 67.01 $\pm$ 0.92 & 231.9 $\pm$  6.3 \\
  & \cellcolor{blue!8} CoreQ
           & \cellcolor{blue!8} 11.50 $\pm$ 0.09
           & \cellcolor{blue!8} 14.76 $\pm$ 0.04
           & \cellcolor{blue!8} \textbf{70.49 $\pm$ 0.93}
           & \cellcolor{blue!8} 566.4 $\pm$ 14.8
           & \cellcolor{blue!8} \textbf{8.40 $\pm$ 0.34}
           & \cellcolor{blue!8} \textbf{10.90 $\pm$ 0.27}
           & \cellcolor{blue!8} 67.95 $\pm$ 2.15
           & \cellcolor{blue!8} 308.8 $\pm$ 10.3 \\
\bottomrule
\end{tabular}%
}
\end{table}

%% file: files/tab-70b.tex
\begin{table}[h]
\centering
\caption{Symmetric 2-bit and 3-bit weight-only quantization on LLaMA-2-70B.
We report perplexity on WikiText-2 and C4 ($\downarrow$) and wall-clock
quantization time (seconds) under per-channel and per-group ($g{=}128$)
granularity. Mean$\,\pm\,$std over 5 calibration seeds. Best per (bits, granularity, metric) in bold.}
\vspace{0.05 in}
\label{tab:l2-70b}
\setlength{\tabcolsep}{4pt}
\resizebox{\linewidth}{!}{%
\begin{tabular}{cl ccc | ccc}
\toprule
& & \multicolumn{3}{c|}{\textbf{Per-channel}} & \multicolumn{3}{c}{\textbf{Per-group ($g{=}128$)}} \\
\cmidrule(lr){3-5} \cmidrule(lr){6-8}
\textbf{Bits} & \textbf{Method}
  & \textbf{Wiki2} ($\downarrow$)
  & \textbf{C4} ($\downarrow$)
  & \textbf{Q.Time (s)}
  & \textbf{Wiki2} ($\downarrow$)
  & \textbf{C4} ($\downarrow$)
  & \textbf{Q.Time (s)} \\
\midrule
FP16 & --- & 3.32 & 5.52 & --- & 3.32 & 5.52 & --- \\
\midrule
\multirow{4}{*}{W2}
  & GPTQ   & 181.31 $\pm$ 19.81 &  83.17 $\pm$  5.14 & 3301.2 $\pm$ 73.0
           &  12.38 $\pm$  1.42 &  11.08 $\pm$  0.10 & 3271.6 $\pm$  9.5 \\
  & GPTAQ  & 251.82 $\pm$ 80.72 &  83.59 $\pm$ 11.38 & 4663.4 $\pm$  3.0
           &  12.28 $\pm$  3.07 &  11.85 $\pm$  2.16 & 4846.5 $\pm$ 43.9 \\
  & LDLQ   & 193.51 $\pm$ 36.21 &  91.32 $\pm$ 10.08 & 2753.5 $\pm$ 26.7
           &   8.68 $\pm$  0.29 &   9.90 $\pm$  0.22 & 2939.7 $\pm$ 76.5 \\
  & \cellcolor{blue!8} CoreQ
           & \cellcolor{blue!8} \textbf{94.97 $\pm$ 12.49}
           & \cellcolor{blue!8} \textbf{38.46 $\pm$  2.37}
           & \cellcolor{blue!8} 4731.9 $\pm$ 69.3
           & \cellcolor{blue!8} \textbf{7.63 $\pm$ 0.49}
           & \cellcolor{blue!8} \textbf{8.95 $\pm$ 0.32}
           & \cellcolor{blue!8} 4772.1 $\pm$ 89.6 \\
\midrule
\multirow{4}{*}{W3}
  & GPTQ   & 5.11 $\pm$ 0.02 & 6.77 $\pm$ 0.02 & 3256.0 $\pm$   0.6
           & 4.00 $\pm$ 0.01 & 5.95 $\pm$ 0.01 & 3323.5 $\pm$  22.6 \\
  & GPTAQ  & 5.09 $\pm$ 0.05 & 6.73 $\pm$ 0.01 & 4727.5 $\pm$  37.1
           & 4.01 $\pm$ 0.02 & 5.95 $\pm$ 0.01 & 4731.5 $\pm$  90.9 \\
  & LDLQ   & 5.06 $\pm$ 0.02 & 6.73 $\pm$ 0.02 & 2855.0 $\pm$ 137.9
           & 3.97 $\pm$ 0.01 & 5.94 $\pm$ 0.00 & 2914.3 $\pm$   0.9 \\
  & \cellcolor{blue!8} CoreQ
           & \cellcolor{blue!8} \textbf{4.83 $\pm$ 0.03}
           & \cellcolor{blue!8} \textbf{6.52 $\pm$ 0.01}
           & \cellcolor{blue!8} 4785.1 $\pm$ 28.3
           & \cellcolor{blue!8} \textbf{3.93 $\pm$ 0.01}
           & \cellcolor{blue!8} \textbf{5.90 $\pm$ 0.00}
           & \cellcolor{blue!8} 4783.5 $\pm$ 96.8 \\
\bottomrule
\end{tabular}%
}
\end{table}

%% file: files/tab-instruct-model.tex
\begin{table}[h]
\centering
\caption{Weight-only quantization of LLaMA-3.1-8B-Instruct on reasoning benchmarks.
We report perplexity on WikiText-2 ($\downarrow$), accuracy on MMLU (5-shot, $\uparrow$), and exact-match accuracy
on GSM8K with chain-of-thought prompting (8-shot, $\uparrow$). Calibrated on 128 samples of C4. Mean$\,\pm\,$std over 3 calibration seeds.
Best result per (bits, granularity, metric) in bold.}
\vspace{0.05 in}
\label{tab:l31-instruct}
\setlength{\tabcolsep}{4pt}
\resizebox{\linewidth}{!}{%
\begin{tabular}{cl ccc | ccc}
\toprule
& & \multicolumn{3}{c|}{\textbf{Per-channel}} & \multicolumn{3}{c}{\textbf{Per-group ($g{=}128$)}} \\
\cmidrule(lr){3-5} \cmidrule(lr){6-8}
\textbf{Bits} & \textbf{Method}
  & \textbf{Wiki2} ($\downarrow$)
  & \textbf{MMLU} ($\uparrow$)
  & \textbf{GSM8K-CoT} ($\uparrow$)
  & \textbf{Wiki2} ($\downarrow$)
  & \textbf{MMLU} ($\uparrow$)
  & \textbf{GSM8K-CoT} ($\uparrow$) \\
\midrule
FP16 & --- & 6.24 & 67.95 & 75.66 & 6.24 & 67.95 & 75.66 \\
\midrule
\multirow{4}{*}{W3}
  & GPTQ   & 154.03 $\pm$ 27.68 & 30.00 $\pm$ 3.07 &  0.28 $\pm$ 0.48
           &  11.06 $\pm$  0.48 & 56.33 $\pm$ 0.51 & 36.52 $\pm$ 1.38 \\
  & GPTAQ  &  33.79 $\pm$ 11.11 & 32.77 $\pm$ 2.31 &  2.86 $\pm$ 1.73
           &   9.95 $\pm$  0.22 & 56.20 $\pm$ 1.11 & 42.53 $\pm$ 3.21 \\
  & LDLQ   &  19.89 $\pm$  1.53 & 39.62 $\pm$ 1.40 &  4.17 $\pm$ 0.92
           &   9.72 $\pm$  0.20 & 59.53 $\pm$ 0.62 & 45.69 $\pm$ 2.25 \\
  & \cellcolor{blue!8} CoreQ
           & \cellcolor{blue!8} \textbf{18.75 $\pm$ 2.98}
           & \cellcolor{blue!8} \textbf{39.76 $\pm$ 2.83}
           & \cellcolor{blue!8} \textbf{8.62 $\pm$ 1.89}
           & \cellcolor{blue!8} \textbf{9.32 $\pm$ 0.14}
           & \cellcolor{blue!8} \textbf{59.95 $\pm$ 0.36}
           & \cellcolor{blue!8} \textbf{48.14 $\pm$ 1.67} \\
\midrule
\multirow{4}{*}{W4}
  & GPTQ   & 188.91 $\pm$ 112.13 & 60.83 $\pm$ 2.28 & 38.19 $\pm$ 32.02
           &   9.60 $\pm$   1.03 & 65.67 $\pm$ 0.64 & 68.56 $\pm$  0.88 \\
  & GPTAQ  &  13.21 $\pm$   3.08 & 62.88 $\pm$ 0.34 & 58.98 $\pm$  3.72
           &   7.67 $\pm$   0.02 & 65.98 $\pm$ 0.10 & 70.08 $\pm$  0.09 \\
  & LDLQ   &   8.58 $\pm$   0.12 & \textbf{64.35 $\pm$ 0.33} & 60.50 $\pm$  1.20
           &   7.64 $\pm$   0.02 & 66.76 $\pm$ 0.19 & 69.87 $\pm$  1.92 \\
  & \cellcolor{blue!8} CoreQ
           & \cellcolor{blue!8} \textbf{8.22 $\pm$ 0.02}
           & \cellcolor{blue!8} 64.22 $\pm$ 0.20
           & \cellcolor{blue!8} \textbf{63.83 $\pm$ 0.92}
           & \cellcolor{blue!8} \textbf{7.58 $\pm$ 0.02}
           & \cellcolor{blue!8} \textbf{66.91 $\pm$ 0.23}
           & \cellcolor{blue!8} \textbf{70.58 $\pm$ 0.80} \\
\bottomrule
\end{tabular}%
}
\end{table}

%% file: appendix/extend-related-work.tex
\newpage
\section{Extended Discussion on Layer-wise Calibration Objectives for Compression}
\label{app:extended-related-work}

\label{app:rw-objective}

A widely used family of compression methods minimizes a quadratic reconstruction error that measures the layer's output distortion under a second-order (Hessian) surrogate \cite{nagel2020up, hubara2021accurate, li2021brecq, frantar2022obq}, with GPTQ being the widely adopted LLM-scale realization \cite{frantar2022gptq}. These methods share a common assumption: each layer's input activation is treated as fixed, and only the layer-local reconstruction error is optimized. Several recent works incorporate information about end-to-end model behavior into this proxy. OAC introduces an output-adaptive calibration objective motivated by output cross-entropy distortion and derives tractable approximations for layerwise optimization \cite{edalati2025oac}. GuidedQuant integrates end-loss gradient information into the objective while modeling cross-weight dependencies within output channels \cite{kimguidedquant}. YAQA targets end-to-end distribution preservation via Kronecker-factored approximations of the full-model KL Hessian and provides adaptive rounding guarantees \cite{tseng2025model}; and BoA exploits attention structure to derive a more informative local objective without backpropagation \cite{kim2024boa}. These methods improve calibration quality but typically incur substantially higher cost, as they require either end-loss gradient computation or more expensive Hessian approximations. 

A complementary line of work addresses the observation that upstream compression introduces distributional shifts in layer activations. To mitigate this, several methods redefine the layer-wise calibration objective to account for activation mismatch. GPTAQ \cite{ligptaq} targets the full mismatch objective—matching the full-precision layer output on clean activations against the quantized layer output on corrupted ones—and derives block-wise closed-form updates. ApiQ \cite{liao2024apiq} adopts a related mismatch objective but embeds it within a learning-based pipeline to jointly optimize quantized weights and LoRA adapters. Closely related to our continuous formulation is QEP \cite{arai2026qep}, which interpolates between standard Hessian-based calibration and full mismatch-aware calibration using a manually tuned, global propagation-strength coefficient. Alternatively, CBQ \cite{ding2023cbq} relaxes strict per-layer independence, quantizing blocks of layers to absorb propagated errors into an expanded optimization unit. Finally, a structurally related motif appears in low-rank SVD compression, where SAES-SVD \cite{hu2026saessvd} balances accumulated upstream error against local distortion via an adaptive alignment coefficient.

{\bf Addressing limitations} CoreQ differs from prior work in two respects. First, unlike GPTAQ, which assumes full mismatch correction is optimal, or QEP, which relies on a rigid global hyperparameter, CoreQ shrinks the correction with a closed-form per-layer coefficient motivated by the finite-calibration analysis in \cref{thm:population-alpha-main}. Second, unlike CBQ or ApiQ, CoreQ requires neither multi-layer optimization units nor gradient-based learning. The corrected objective retains the standard Hessian-weighted quadratic form, so the bounded search of $K$-CoreQ applies directly with negligible overhead relative to standard local PTQ.